%% file: OneDimBO.tex
\documentclass[english]{article}

\input{preamble.tex}

\usepackage{natbib}
\usepackage{microtype}
\usepackage{graphicx}
\usepackage{subcaption}
\usepackage{booktabs} % for professional tables
\usepackage{multicol}
\usepackage{natbib}

\usepackage{hyperref}

\usepackage[accepted]{icml2018}

\icmltitlerunning{Tight Regret Bounds for Bayesian Optimization in One Dimension}

\begin{document}

\twocolumn[
\icmltitle{Tight Regret Bounds for Bayesian Optimization in One Dimension}


% Academic affiliations should list Department, University, biby, Region, Country


\begin{icmlauthorlist}
\icmlauthor{Jonathan Scarlett}{to}
\end{icmlauthorlist}

\icmlaffiliation{to}{Department of Computer Science \& Department of Mathematics, National University of Singapore}

\icmlcorrespondingauthor{Jonathan Scarlett}{scarlett@comp.nus.edu.sg}

\icmlkeywords{Machine Learning, ICML}

\vskip 0.3in
]



\printAffiliationsAndNotice{}  % leave blank if no need to mention equal contribution

\begin{abstract}
    We consider the problem of Bayesian optimization (BO) in one dimension, under a Gaussian process prior and Gaussian sampling noise.  We provide a theoretical analysis showing that, under fairly mild technical assumptions on the kernel, the best possible cumulative regret up to time $T$ behaves as $\Omega(\sqrt{T})$ and $O(\sqrt{T\log T})$. This gives a tight characterization up to a $\sqrt{\log T}$ factor, and includes the first non-trivial lower bound for noisy BO.  Our assumptions are satisfied, for example, by the squared exponential and Mat\'ern-$\nu$ kernels, with the latter requiring $\nu > 2$.  Our results certify the near-optimality of existing bounds (Srinivas {\em et al.}, 2009) for the SE kernel, while proving them to be strictly suboptimal for the Mat\'ern kernel with $\nu > 2$.  % Another important implication in the latter setting is that the Bayesian problem is strictly easier than its non-Bayesian counterpart for functions with a bounded RKHS norm.
\end{abstract}

\section{Introduction}

Bayesian optimization (BO) \cite{Sha16} is a powerful and versatile tool for black-box function optimization, with applications including parameter tuning, robotics, molecular design, sensor networks, and more.  The idea is to model the unknown function as a Gaussian process with a given {\em kernel function} dictating the smoothness properties.  This model is updated using (typically noisy) samples, which are selected to steer towards the function maximum.

One of the most attractive properties of BO is its efficiency in terms of the number of function samples used.  Consequently, algorithms with {\em rigorous guarantees} on the trade-off between samples and optimization performance are particularly valuable.  Perhaps the most prominent work in the literature giving such guarantees is that of \cite{Sri09}, who consider the {\em cumulative regret}:
\begin{equation}
R_T = \sum_{t=1}^T \Big( \max_{x} f(x) - f(x_t) \Big), \label{eq:cumul}
\end{equation}
where $f$ is the function being optimized, and $x_t$ is the point chosen at time $t$.  Under a Gaussian process (GP) prior and Gaussian noise, it is shown in \cite{Sri09} that an algorithm called Gaussian Process Upper Confidence Bound (GP-UCB) achieves a cumulative regret of the form
\begin{equation}
R_T = O^*( \sqrt{T\gamma_T}), \label{eq:Srinivas}
\end{equation}  
where $\gamma_T = \max_{x_1,\dotsc,x_T} I( \fv; \yv )$ (with function values $\fv = (f(x_1),\dotsc,f(x_T))$ and noisy samples $\yv = (y_1,\dotsc,y_T)$) is known as the {\em maximum information gain}.  Here $I( \fv; \yv )$ denotes the mutual information \cite{Cov01} between the function values and noisy samples, and $O^*(\cdot)$ denotes asymptotic notation up to logarithmic factors.

The guarantee \eqref{eq:Srinivas} ensures sub-linear cumulative regret for many kernels of interest.  However, the literature is severely lacking in {\em algorithm-independent lower bounds}, and without these, it is impossible to know to what extent the upper bounds, including  \eqref{eq:Srinivas}, can be improved.  In this work, we address this gap in detail in the special case of a one-dimensional function.  We show that the best possible cumulative regret behaves as $\Theta^*(\sqrt{T})$ under mild assumptions on the kernel, thus identifying both cases where \eqref{eq:Srinivas} is near-optimal, and cases where it is strictly suboptimal.

\subsection{Related Work} \label{sec:previous_work}

An extensive range of BO algorithms have been proposed in the literature, typically involving the maximization of an acquisition function \cite{Hen12,Her14,Rus14a,Wan16}; see \cite{Sha16} for a recent overview.  
As mentioned above, the most relevant algorithm to this work for the noisy setting is GP-UCB \cite{Sri09}, which constructs {\em confidence bounds} in which the function lies with high probability, and samples the point with the highest upper confidence bound.  Several extensions to GP-UCB have also been proposed, including contextual \cite{Kra11,Bog16}, batch \cite{Con13,Des14a}, and high-dimensional \cite{Kan15,Rol18} variants.

In the noiseless setting, it has been shown that it is possible to achieve {\em bounded} cumulative regret \cite{Fre12,Kaw15} under some technical assumptions.  In \cite{Fre12}, this is done by keeping track of a set of potential maximizers, and sampling increasingly finely in order to shrink that set and ``zoom in'' towards the optimal point.  Similar ideas have also been used in the noisy setting for studying batch variants of GP-UCB \cite{Con13}, simultaneous online optimization (SOO) methods \cite{Wan14f}, and lookahead algorithms that use confidence bounds \cite{Bog16a}.  Returning to the noiseless setting, upper and lower bounds were given in \cite{Gru10} for kernels satisfying certain smoothness assumptions, with the lower bounds showing that bounded cumulative regret is not always to be expected.

Alongside the Bayesian view of the Gaussian process model, several works have also considered a {\em non-Bayesian} counterpart assuming that the function has a bounded norm in the associated reproducing kernel Hilbert space (RKHS).  Interestingly, GP-UCB still provides similar guarantees to \eqref{eq:Srinivas} in this setting \cite{Sri09}.  Moreover, lower bounds have been proved; see \cite{Bul11} for the noiseless setting, and \cite{Sca17a} for the noisy setting.  In the latter, the lower bounds nearly match the GP-UCB upper bound for the squared exponential (SE) kernel, but gaps remain for the Mat\'ern kernel.  For reference, we note that these kernels are defined as follows:
\begin{align}
k_{\text{SE}}(x,x') &= \exp \bigg(- \dfrac{\|x - x'\|^2}{2l^2} \bigg) \\ 
k_{\text{Mat\'ern}}(x,x') &= \dfrac{2^{1-\nu}}{\Gamma(\nu)} \bigg(\dfrac{\sqrt{2\nu}\|x - x'\|}{l}\bigg)^{\nu} \nonumber \\
    &\qquad\times  B_{\nu}\bigg(\dfrac{\sqrt{2 \nu}\|x - x'\|}{l} \bigg),
\end{align}
where $l>0$ is a lengthscale parameter, $\nu > 0$ is a smoothness parameter, $B_{\nu}$ is the modified Bessel function, and $\Gamma$ is the gamma function.

The multi-armed bandit (MAB) \cite{Bub12} literature has developed alongside the BO literature, with the two often bearing similar concepts.  The MAB literature is far too extensive to cover here, but it is worth mentioning that sharp lower bounds are known in numerous settings \cite{Bub12}, and the above-mentioned concept of ``zooming in'' to the optimal point has also been explored \cite{Kle08}.  To our knowledge, however, none of the existing MAB results are closely related to our own.

\subsection{Our Results and Their Implications} \label{sec:contributions}

The main results of this paper are informally summarized as follows.

{\bf Main Results (Informal).} {\em Under mild technical assumptions on the kernel, satisfied (for example) by the SE kernel and Mat\'ern-$\nu$ kernel with $\nu > 2$, the best possible cumulative regret of noisy BO in one dimension behaves as $\Omega(\sqrt{T})$ and $O(\sqrt{T \log T} )$. }

Our results have several important implications:
\begin{itemize}
    \item To our knowledge, our lower bound is the first of any kind in the noisy Bayesian setting, and is tight up to a $\sqrt{\log T}$ factor under our technical assumptions.
    \item Our lower bound also establishes the order-optimality of the $O^*(\sqrt{T})$ upper bound of \cite{Sri09} applied to the SE kernel, up to logarithmic factors.
    \item On the other hand, our upper bound establishes that the upper bound of \cite{Sri09} for the Mat\'ern-$\nu$ kernel, namely $O^*( T^{\frac{\nu+2}{2\nu+2}})$, is strictly suboptimal for $\nu > 2$.  For example, if $\nu = 3$, then this is $O^*(T^{0.625})$, as opposed to our upper bound of $O^*(T^{0.5})$.  (See also \cite{She17} for recent improvements over \cite{Sri09} under the Mat\'ern kernel in higher dimensions and/or with smaller $\nu$).
    \item Another important implication for the Mat\'ern kernel with $\nu > 2$ is that the Bayesian setting is provably less difficult than the non-Bayesian RKHS counterpart; the latter has cumulative regret $\Omega( T^{\frac{\nu+1}{2\nu+1}} )$ \cite{Sca17a}, which is strictly worse than $O(\sqrt{T \log T})$.
\end{itemize}

Our upper bound is stated formally in Section \ref{sec:ub}, and its technical assumptions are given in Section \ref{sec:bo_setup}.  We build on the ideas of \cite{Fre12} for the noiseless setting, while addressing highly non-trivial challenges arising in the presence of noise.  

Our lower bound is stated formally in Section \ref{sec:lb}, and its technical assumptions are given in Section \ref{sec:bo_setup}.  The analysis is based on a reduction to binary hypothesis testing and an application of Fano's inequality \cite{Cov01}.  This approach is inspired by previous work on lower bounds for stochastic convex optimization \cite{Rag11}, but the details are very different.

% While our focus on the one-dimensional setting may seem restrictive, this is mainly done to simplify the proofs.  As we discuss in the supplementary material, our analysis and results extend to {\em any constant dimension} $d$ that is held fixed as $T \to \infty$, though the dependence on $d$ in the constant factors is very different in the upper and lower bounds.

\section{Problem Setup} 

\subsection{Bayesian Optimization} \label{sec:bo_setup}

We seek to sequentially optimize an unknown reward function $f(x)$ over the one-dimensional domain $D = [0,1]$; note that any interval can be transformed to this choice via re-scaling.  At time $t$, we query a single point $x_t \in D$ and observe a noisy sample $y_t = f(x_t) + z_t$, where $z_t \sim N(0,\sigma^2)$ for some noise variance $\sigma^2 > 0$, with independence across different times.  We measure the performance using the cumulative regret $R_T$, defined in \eqref{eq:cumul}.

We henceforth assume $f$ to be distributed according to Gaussian process (GP) \cite{Ras06} having mean zero and kernel function $k(x,x')$. The posterior distribution of $f$ given the points $\xv_t =[x_1,\dotsc,x_t]^T$ and observations $\yv_t =[y_1,\dotsc,y_t]^T$ up to time $t$ is again a GP, with the posterior mean and variance given by \cite{Ras06}
\begin{align}
\mu_{t}(x) &= \kv_t(x)^T\big(\Kv_t + \sigma^2 \Iv_t \big)^{-1} \yv_t \label{eq:mu_update} \\
\sigma_{t}(x)^2 &= k(x,x) - \kv_t(x)^T \big(\Kv_t + \sigma^2 \Iv_t \big)^{-1} \kv_t(x), \label{eq:sigma_update}  
\end{align}
where $\kv_t(x) = \big[k(x_i,x)\big]_{i=1}^t$, $\Kv_t = \big[k(x_i,x_j)\big]_{i,j}$, and $\Iv_t$ is the $t \times t$ identity matrix. 

\subsection{Technical Assumptions}

Here we introduce several assumptions that will be adopted in our main results, some of which were also used in the noiseless setting \cite{Fre12}.

\begin{assump} \label{as:kernel_basic}
    We have the following:
    \begin{enumerate}
        \item The kernel $k$ is stationary, depending on its inputs $(x,x')$ only through $\tau = x-x'$;
        \item The kernel $k$ satisfies $k(x,x') \le 1$ for all $(x,x')$, and $k(x,x) = 1$ for all $x\in D$;
        % \item $k$ is four times differentiable as a function of $\tau$, and the forth derivative at $\tau=0$ is upper bounded by $Q^2$ for some constant $Q > 0$.
    \end{enumerate}
\end{assump}

Given the stationarity assumption, the assumptions $k(x,x') \le 1$ and $k(x,x) = 1$ are without loss of generality, as one can always re-scale the function and adjust the noise variance $\sigma^2$ accordingly.  
% The assumption $k(x,x') \ge 0$ means that two given function values are always positively correlated, and is true for most widely-adopted stationary kernels such as SE and Mat\'ern.  {\todo Check if used}% The four times differentiability assumption was also used in the noiseless case \cite{Fre12}, and we require it in order to apply a technical result given therein (see Lemma \ref{lem:var_bound} in Section \ref{sec:ub}).

Next, we give some high-probability assumptions on the random function $f$ itself.

\begin{assump} \label{as:kernel_diff}
    There exists a constant $\delta_1 \in (0,1)$ such that, with probability at least $1 - \delta_1$, we have the following:
    \begin{enumerate}
        \item The function $f$ has a unique maximizer $x^*$ such that
        \begin{equation}
        f(x^*) \ge f(x') + \epsilon \label{eq:eta}
        \end{equation}
        for any {\em local maximum} $x'$ that differs from $x^*$, for some constant $\epsilon > 0$.
        \item The function $f$ is twice differentiable;
        \item The function $f$ and its first two derivatives are bounded:
        \begin{equation}
        |f(x)| \le c_0, \quad |f'(x)| \le c_1, \quad |f''(x)| \le c_2  \label{eq:c0}
        \end{equation}
        for all $x \in D$ and some constants $(c_0,c_1,c_2)$.  This implies that $f$ is $c_1$-Lipschitz continuous, and $f'$ is $c_2$-Lipschitz continuous.
    \end{enumerate}
\end{assump}

The assumption of a unique maximizer holds with probability one in most non-trivial cases \cite{Fre12}, and \eqref{eq:eta} simply formally defines the gap to the second-highest peak.  Moreover, given twice differentiability, the remaining conditions in \eqref{eq:c0} are very mild, only requiring that the function value and its derivatives are bounded, and formally defining the corresponding constants.

% with \eqref{eq:c1} only requiring the first derivatives are bounded, and \eqref{eq:Taylor_general} representing a standard Taylor expansion.  

Next, we provide assumptions regarding the derivatives of $f$ and the resulting Taylor expansions (typically around the optimizer $x^*$).  We adopt slightly different assumptions for the upper and lower bounds, starting with the former.

\begin{assump} \label{as:taylor}
    There exist constants $\delta_2 \in (0,1)$ and $\rho_0 \in \big(0,\frac{1}{2}\big)$ such that conditioned on the events in Assumption \ref{as:kernel_diff}, we have with probability at least $1 - \delta_2$ that {\em one of the following} is true:
    
    \begin{enumerate}
        \item The maximizer is at an endpoint (i.e., $x^* = 0$ or $x^* = 1$), and satisfies the following {\em locally linear} behavior: For all $\xi \in [0,\rho_0]$ (if $x^* = 0$) or $\xi \in [-\rho_0,0]$ (if $x^* = 1$), it holds that
        \begin{equation}
        f(x^*) -  \cunder_1 |\xi| \ge f(x^* + \xi) \ge f(x^*) - \cbar_1 |\xi| \label{eq:linear_opt}
        \end{equation}
        for some constants $\cbar_1 \ge \cunder_1 > 0$.
        \item The maximizer satisfies $x^* \in (\rho_0,1-\rho_0)$, and $f$ satisfies the following {\em locally quadratic} behavior: For all $\xi \in [-\rho_0,\rho_0]$, we have
        \begin{equation}
        f(x^*) - \cunder_2\xi^2 \ge f(x^*+\xi) \ge f(x^*) - \cbar_2\xi^2 \label{eq:Taylor_opt}
        \end{equation} 
        for some constants $\cbar_2 \ge \cunder_2 > 0$.
    \end{enumerate}
\end{assump}

This assumption is near-identical to the main assumption adopted in the noiseless setting \cite{Fre12}, and is also mild given the assumption of twice differentiability.  Indeed, \eqref{eq:linear_opt} and \eqref{eq:Taylor_opt} amount to standard Taylor expansions, with the assumptions $\cunder_1 > 0$ and $\cunder_2> 0$ only requiring non-vanishing gradient at the endpoint (first case) or non-vanishing second derivative at the function maximizer (second case).  These conditions typically hold with probability one \cite{Fre12}.  

% In the lower bound we will use \eqref{eq:Taylor_general} for multiple values of $x$, whereas in the upper bound we will only use it for $x = x^*$, the unique maximizer of $f$.  If $x^* \in (0,1)$ (i.e., $x^*$ is not a boundary point), then we must have $f'(x^*) = 0$, and \eqref{eq:Taylor_general} simplifies to for $\Delta \in [-\rho_0,\rho_0]$.  This means that $f$ is {\em locally quadratic} near its maximum, i.e., it satisfies quadratic upper and lower bounds for points that are within a distance $\rho_0$ of $x^*$.  {\todo Trivial if $\cbar_2 > 0$?}

The following assumption will be used for the lower bound.

\begin{assump} \label{as:endpoints}
    There exists constants $\delta'_2 \in (0,1)$ and $\rho_0 \in \big(0,\frac{1}{2}\big)$ such that conditioned on the events in Assumption \ref{as:kernel_diff},  {\em both of the following} hold with probability at least $1 - \delta'_2$:
    
    \begin{enumerate}
        \item For any $x \in D$ and $\xi \in [-\rho_0,\rho_0]$ for which $x + \xi \in D$, we have
        \begin{equation}
        \xi \cdot f'(x) + \cunder'_2\xi^2 \le f(x+\xi) - f(x) \le \xi \cdot f'(x) + \cbar'_2\xi^2. \label{eq:Taylor_general}
        \end{equation}
        for some (possibly negative) constants $\cbar'_2$, $\cunder'_2$.
        \item The maximizer satisfies $x^* \in (\rho_0,1-\rho_0)$, and $f$ satisfies the following for all $\xi \in [-\rho_0,\rho_0]$:
        \begin{align}
        f(x^*) - \cunder_2\xi^2 \ge f(x^*+\xi) \ge f(x^*) - \cbar_2\xi^2 \label{eq:Taylor_opt2}
        \end{align} 
        for some constants $\cbar_2 \ge \cunder_2 > 0$.
    \end{enumerate}
\end{assump}

\begin{figure}
    \centering
    \includegraphics[width=0.47\textwidth]{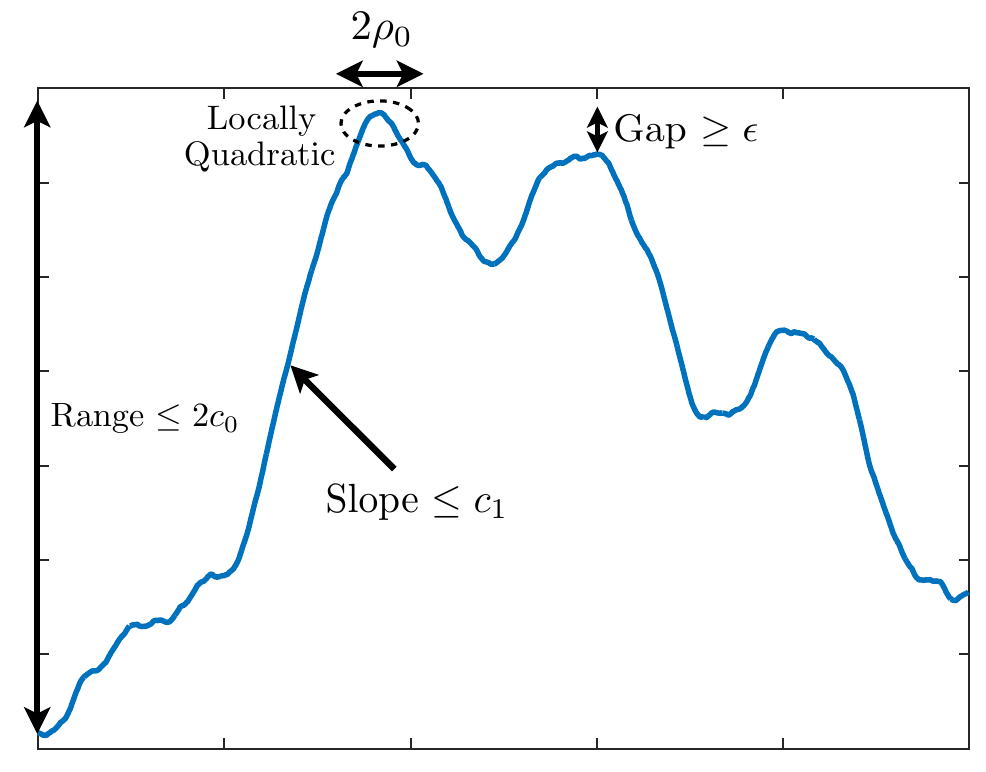}
    \par
    \caption{Illustration of some of the main assumptions: The function is bounded within $[-c_0,c_0]$ and its derivative within $[-c_1,c_1]$, the gap to the second highest peak is at least $\epsilon$, and the function is locally quadratic for points within a distance $\rho_0$ of the maximizer. \label{fig:assumptions}}
    \vspace*{-3ex}
\end{figure}

The first part is similar to \eqref{eq:Taylor_opt}, but performs a Taylor expansion around an arbitrary point rather than the specific point $x^*$, and the second part is precisely \eqref{eq:Taylor_opt}.  Note, however, that here we are assuming {\em both} of two conditions to hold, rather than one of two.  Hence, we are implicitly assuming that the first item of Assumption \ref{as:taylor} does {\em not} have a significant probability of occurring.  For stationary kernels, the only situations where an endpoint has a high probability of being optimal are those where $f$ varies very slowly (e.g., the SE kernel with a larger lengthscale than the domain width).  Such functions are of limited practical interest.

Similarly to the noiseless setting \cite{Fre12}, all of the above assumptions hold for the SE kernel, as well as the Mat\'ern-$\nu$ kernel with $\nu > 2$, with the added caveat that $\delta'_2$ in Assumption \ref{as:endpoints} is a function of the lengthscale and cannot be chosen arbitrarily.  Specifically, a smaller lengthscale implies a smaller value of $\delta'_2$.  In contrast, $\delta_1$ and $\delta_2$ in Assumptions \ref{as:kernel_diff} and \ref{as:taylor} can be made arbitrary small by suitably changing the constants $\epsilon$, $c_0$, $c_1$, $c_2$, $\rho_0$, and so on.

An illustration of some of the main assumptions and their associated constants is given in Figure \ref{fig:assumptions}.

\section{Upper Bound} \label{sec:ub}

Our upper bound is formally stated as follows.

\begin{thm} \label{thm:ub}
    {\em (Upper Bound)}
    Consider the problem of BO in one dimension described in Section \ref{sec:bo_setup}, with time horizon $T$ and noise variance $\sigma^2$ satisfying $\sigma^2 \ge \frac{c_{\sigma}}{T^{1-\zeta}}$ for some $c_{\sigma} > 0$ and $\zeta > 0$.  Under Assumptions \ref{as:kernel_basic}, \ref{as:kernel_diff}, and \ref{as:taylor}, there exists an algorithm satisfying the following: With probability at least $1 - \delta_1 - \delta_2$ (with respect to the Gaussian process $f$), the average cumulative regret (averaged over the noisy samples) satisfies
    \begin{equation}
    \EE[R_T] \le C \Big( 1 + \sigma \sqrt{T \log T}\Big). \label{eq:ub}
    \end{equation}
    Here $\delta_1$ and $\delta_2$ are defined in Assumptions \ref{as:kernel_diff} and \ref{as:taylor}, and $C$ depends only on the constants therein and $(c_{\sigma},\zeta)$.
\end{thm}

The assumption that $\sigma^2 \ge \frac{c_{\sigma}}{T^{1-\zeta}}$ for some $(c_{\sigma},\zeta)$ is very mild, since typically $\sigma^2$ is constant with respect to $T$.  The proof of Theorem \ref{thm:ub} extends immediately to a high probability guarantee with respect to both $f$ and the noisy samples (i.e., holding with probability $1 - \delta_1 - \delta_2 - \delta$ for $\delta$ in Lemma \ref{lem:conf_bounds} below).  We have stated the above form for consistency with the lower bound, which will be given in Section \ref{sec:lb}.

\subsection{High-Level Description of the Algorithm}

The algorithm considered in the proof of Theorem \ref{thm:ub} is described informally in Algorithm \ref{alg:resample}; the details will be established throughout the proof of Theorem \ref{thm:ub}, and a complete description is given in Appendix \ref{sec:pf_ub}.

\begin{algorithm} 
    \caption{Informal description of our algorithm, based on reducing uncertainty in epochs via repeated sampling.} \label{alg:resample}
    \begin{algorithmic}[1]
        \REQUIRE Domain $D$, GP prior ($\mu_0$, $k_0$), discrete sub-domain $\Lc \subseteq D$, time horizon $T$.
        \STATE Initialize $t=1$, epoch number $i=1$, potential maximizers $M_{(0)}=\Lc$, and target confidence $\eta_{(0)}$.
        \WHILE{less than $T$ samples have been taken}
        \STATE Set $\eta_{(i)} = \frac{1}{2} \eta_{(i-1)}$.
        \STATE Sample each point within a subset $\Lc_{(i)} \subseteq \Lc$ repeatedly $K_{(i)}$ times, where $\Lc_{(i)}$ and $K_{(i)}$ are chosen such that after this sampling, all points $x \in M_{(i-1)}$ satisfy upper and lower confidence bounds of the form
        $$ \LCB_t(x) \le f(x) \le \UCB_t(x), $$ 
        with the gap between the two bounded by $|\UCB_t(x) - \LCB_t(x)| \le 2\eta_{(i)}$.  % \\ (Increment $t$ each time a sample is taken)
        \smallskip 
        \STATE Update the set of potential maximizers:
        \begin{multline*}
        M_{(i)} = \Big\{ x \in M_{(i-1)} \,:\, \\ \UCB_t(x) \ge \max_{x' \in \in M_{(i-1)}} \LCB_t(x') \Big\}.
        \end{multline*}
        % \STATEx $\quad$~ where $\UCB_t$ and $\LCB_t$ are suitably-chosen upper and lower confidence bounds.
        \STATE Increment $i$.
        \ENDWHILE
    \end{algorithmic}
\end{algorithm}

 As in the noiseless setting \cite{Fre12}, the idea is to operate in {\em epochs} and sample a set of increasingly closely-packed points $\Lc_{(i)}$ to reduce the posterior variance, but only within a set of {\em potential maximizers} that are updated according to the confidence bounds.  As a simple means of bringing the effective noise level down, we perform {\em resampling}, i.e., sampling the same point $K_{(i)}$ times consecutively.  In each epoch, we sample enough to be able to produce upper and lower confidence bounds $\UCB_t(x)$ and $\LCB_t(x)$  that differ by at most a target value $2\eta_{(i)}$ within $M_{(i-1)}$, and then the target is halved for the next epoch.

We do not expect our algorithm to perform well in practice by any means, but it still suffices for our purposes in establishing $O(\sqrt{T \log T})$ regret.  Indeed, we have made no attempt to optimize the corresponding constant factors, and doing so would require more sophisticated techniques.  Moreover, the quantities $\Lc_{(i)}$, $K_{(i)}$, $\UCB_t$, and $\LCB_t$ in Algorithm \ref{alg:resample} are chosen as functions of both the kernel and the constants appearing in our assumptions, which limits the algorithm's practical utility even further.  Note, however, that these constants are merely a function of the kernel, and that suitable bounds suffice in place of exact values (e.g., lower bound on $\rho_0$, upper bound on $c_0$, etc.).

While our algorithm assumes a known time horizon $T$ (which is used when selecting $K_{(i)}$; see Appendix \ref{sec:pf_ub}), this assumption can easily be dropped via a standard doubling trick.  The details are given in Appendix \ref{sec:doubling}.

\subsection{Auxiliary Lemmas}

%In the following, we let  $\sigmatil_t(x) = k(x,x) - \kv_t(x)^T \Kv_t^{-1} \kv_t(x)$ denote the posterior variance, as per \eqref{eq:sigma_update}, in the case that we observe noiseless samples, i.e. $\sigma^2$ is replaced by zero.  The following result was shown in \cite{Fre12}.
%
%\begin{lem} \label{lem:var_bound}
%    {\em \cite[Proposition 4]{Fre12}}
%    Fix $\delta > 0$.  Under Assumption \ref{as:kernel_basic}, if we observe noiseless samples of $f$ at a set of locations $\{x_1,\dotsc,x_t\} \subseteq D$ such that every point in $D$ is within a distance $\delta$ of at least one point in $\{x_1,\dotsc,x_t\}$, then the posterior variance satisfies
%    \begin{equation}
%        \max_{x \in D} \sigmatil_t(x) \le \frac{Q \delta^2}{4},
%    \end{equation}
%    where $Q$ is defined in Assumption \ref{as:kernel_basic}.
%\end{lem}
%
%In order to apply this result in our noisy setting, we use the following lemma relating the noiseless and noisy posterior variance.  To the best of our knowledge, this result is novel, and may be of independent interest.
%
%\begin{lem} \label{lem:noiseless_to_noisy}
%    Under Assumption \ref{as:kernel_basic}, for any set of sampled points $(x_1,\dotsc,x_t)$ and query point $x \in D$, the posterior variance for noisy samples $\sigma_t^2(x)$ and the posterior variance for noiseless samples $\sigmatil_t^2(x)$ are related as follows:
%    \begin{equation}
%        \sigma_t^2(x) \le \sigmatil_t^2(x) + \sigma^2,
%    \end{equation}
%    where $\sigma^2$ is the noise variance.
%\end{lem}
%\begin{proof}
%    See Appendix \ref{sec:pf_noiseless_to_noisy}
%\end{proof}

Here we present two very standard auxiliary lemmas.  We begin with a simpler version of the conditions of Srinivas {\em et al.} \cite{Sri09} guaranteeing that the posterior mean and variance provide valid confidence bounds with high probability.\footnote{{\bf Correction:} Lemma \ref{lem:conf_bounds} is no longer used in a corrected version of the proof of Theorem \ref{thm:ub}; see Section \ref{sec:errata} for details.}  The reason for being slightly simpler is that we are considering a fixed time horizon.

\begin{lem} \label{lem:conf_bounds}
    Fix $\delta \in (0,1)$.  For any finite set of points $\Lc \subseteq D$ and time horizon $T$, under the choice $\beta_T = 2\log\frac{|\Lc| \cdot T}{\delta}$, it holds that
    \begin{equation}
    |f(x) - \mu_{t}(x)| \le \beta_T^{1/2} \sigma_t(x), \quad \forall x\in\Lc,~ t=1,\dotsc,T, \label{eq:conf_bounds}
    \end{equation}
    with probability at least $1 - \delta$.
\end{lem}
\begin{proof}
    % See Appendix \ref{sec:pf_conf_bounds}.
    It was shown in \cite{Sri09} that for fixed $x$ and $t$, the event $|f(x) - \mu_{t}(x)| \le \beta_T^{1/2} \sigma_t(x)$ holds with probability at least $1-e^{-\beta_T / 2}$.  the lemma follows by substituting the choice of $\beta_T$ and taking the union bound over the $|\Lc| \cdot T$ values of $x$ and $t$.
\end{proof}

The following lemma is also standard, and has been used (implicitly or explicitly) in the study of multiple algorithms that eliminate suboptimal points based on confidence bounds \cite{Fre12,Con13,Bog16a}.  For completeness, we give a short proof.

\begin{lem} \label{lem:elimination}
    Suppose that at time $t$, for all $x$ within a set of points $\Lctil \subseteq D$, it holds that
    \begin{equation}
    \LCB_t(x) \le f(x) \le \UCB_t(x) \label{eq:conf_eta}
    \end{equation}
    for some bounds $\UCB_t$ and $\LCB_t$ such that 
    \begin{equation}
    \max_{x \in \Lctil} \big| \UCB_t(x) - \LCB_t(x) \big| \le 2\eta. \label{eq:cb_assump}
    \end{equation}
    Then any point $x \in \Lctil$ satisfying $f(x) < \max_{x' \in \Lctil} f(x') - 4\eta$ must also satisfy
    \begin{equation}
    \UCB_t(x) < \max_{x \in \Lctil} \LCB_t(x). 
    \end{equation}
    That is, any $4\eta$-suboptimal point can be ruled out according to the confidence bounds \eqref{eq:conf_eta}.
\end{lem}
\begin{proof}
    % See Appendix \ref{sec:pf_elimination}.
    We have
    \begin{align}
    \UCB_t(x)
    &\le \LCB_t(x) + 2\eta \label{eq:conf_prop_pf0} \\
    &\le f(x) + 2\eta \label{eq:conf_prop_pf1} \\
    &< \max_{x' \in \Lctil} f(x') - 2\eta \label{eq:conf_prop_pf2} \\
    &\le \max_{x' \in \Lctil} \UCB_t(x')  - 2\eta \label{eq:conf_prop_pf3} \\
    &\le \max_{x' \in \Lctil} \LCB_t(x'),  \label{eq:conf_prop_pf4}
    \end{align}
    where \eqref{eq:conf_prop_pf0} and \eqref{eq:conf_prop_pf4} follow from \eqref{eq:cb_assump},  \eqref{eq:conf_prop_pf1} and \eqref{eq:conf_prop_pf3} follow from the confidence bounds in \eqref{eq:conf_eta}, and \eqref{eq:conf_prop_pf2} follows from the assumption $f(x) < \max_{x' \in \Lctil} f(x') - 4\eta$.
\end{proof}

\subsection{Outline of Proof of Theorem \ref{thm:ub}}

Here we provide a high-level outline of the Proof of Theorem \ref{thm:ub}; the details are given in Appendix \ref{sec:pf_ub}.

Algorithm \ref{alg:resample} only samples on a discrete sub-domain $\Lc$.  This set is chosen to be a set of regularly-spaced points that are fine enough to ensure that the cumulative regret with respect to $\max_{x \in \Lc} f(x)$ is within a constant value of the cumulative regret with respect to $\max_{x \in D} f(x)$.  Working with the finite set $\Lc$ helps to simplify the subsequent analysis.

We split the epochs into two classes, which we call {\em early epochs} and {\em late epochs}.  The late epochs are those in which we have shrunk the potential maximizers down enough to be entirely within the locally quadratic region, {\em cf.}, Figure \ref{fig:assumptions}; here we only discuss the second case of Assumption \ref{as:taylor}, which is the more interesting of the two.  Since the width of the locally quadratic region is constant, we can show that this occurs after a {\em finite number of epochs}, each lasting for at most $O(\log T)$ time.  Hence, even if we naively upper bound the instant regret by $2c_0$ according to \eqref{eq:c0}, the overall regret incurred within the early epochs is insignificant.

In the later epochs, we exploit the locally quadratic behavior to show that the set of potential maximizers shrinks rapidly, i.e., by a constant factor after each epoch.  As a result, we can let the repeatedly-sampled set $\Lc_{(i)}$ in Algorithm \ref{alg:resample} lie within a given interval that similarly shrinks, thereby controlling the number of samples we need to take in the epoch.  % This number turns out to behave as $O\big( 1 + \frac{\sigma^2\log T}{ \eta_{(i)}^2 } \big)$.

By Lemma \ref{lem:elimination}, after we attain uniform $\eta_{(i)}$-confidence, the instant regret incurred at each time thereafter is at most $4\eta_{(i)}$.  Using the fact that $\eta_{(i)} = \eta_{(0)} 2^{-i}$ and summing over the epochs, we find that the overall regret behaves as in \eqref{eq:ub}.

A notable difficulty that we omitted above is how we attain the confidence bounds in order to update the potential maximizers $M_{(i)}$.  While we directly apply Lemma \ref{lem:conf_bounds} for the points that were repeatedly sampled, we found it difficult to do this for the non-sampled points.  For those, we instead use Lipschitz properties of the function.  In the early epochs, we use the global Lipschitz constant $c_1$ from Assumption \ref{as:kernel_diff}, whereas in the later epochs, we find a considerably smaller Lipschitz constant due to the locally quadratic behavior.

\section{Lower Bound} \label{sec:lb}

Our lower bound is formally stated as follows. % and matches the upper bound in Theorem \ref{thm:ub} up to a $\sqrt{\log T}$ factor. 

\begin{thm} \label{thm:lb}
    {\em (Lower Bound)}
    Consider the one-dimensional BO problem from Section \ref{sec:bo_setup}, with time horizon $T$ and noise variance $\sigma^2$ satisfying $\sigma^2 \le c'_{\sigma} T^{1-\zeta'}$ for some $c'_{\sigma} > 0$ and $\zeta' > 0$.  Under Assumptions \ref{as:kernel_basic}, \ref{as:kernel_diff}, and \ref{as:endpoints}, any algorithm must yield the following:  With probability at least\footnote{{\bf Correction:} The current proof leads to a probability of $(1-\delta_1-\delta'_2)/2$ rather than $1-\delta_1-\delta'_2$; see Section \ref{sec:errata} for details.} $1-\delta_1-\delta'_2$ (with respect to the Gaussian process $f$), the average cumulative regret (averaged over the noisy samples) satisfies
    \begin{equation}
    \EE[R_T] \ge C'\Big( 1+ \sigma \sqrt{T} \Big). \label{eq:lb}
    \end{equation}
    Here $\delta_1$ and $\delta'_2$ are defined in Assumptions \ref{as:kernel_diff} and \ref{as:endpoints}, and $C'$ depends only on the constants therein and $(c'_{\sigma},\zeta')$.
\end{thm}

The assumption that $\sigma^2 \le c'_{\sigma} T^{1-\zeta'}$ for some $(c'_{\sigma},\zeta')$ is very mild, since typically $\sigma^2$ is constant with respect to $T$.  The assumption is required to avoid \eqref{eq:lb} contradicting the trivial $O(T)$ upper bound.  We also note that Theorem \ref{thm:lb} immediately implies an $\Omega\big( 1+ \sigma \sqrt{T} \big)$ lower bound on the expected regret $\EE[R_T]$ with respect to {\em both} $f$ and the noisy samples, as long as $1 - \delta_1 - \delta'_2 > 0$.  As discussed following Assumption \ref{as:endpoints}, the latter condition is mild.

In the remainder of the section, we introduce some of the main tools and ideas, and then outline the proof.  We note that $\EE[R_T] = \Omega(1)$ is trivial, as the average regret of the first sample alone is lower bounded by a constant.  As a result, we only need to show that $\EE[R_T] = \Omega( \sigma\sqrt{T} )$.
%Similarly, the case $\sigma^2 = O\big( \frac{1}{T}\big)$ reduces \eqref{eq:lb} to $\EE[R_T] = \Omega(1)$, which is trivial.  As a result, we only need to show that $\EE[R_T] = \Omega( \sigma\sqrt{T} )$ in the case that $\sigma^2 = \omega\big( \frac{1}{T} \big)$.

\subsection{Reduction to Binary Hypothesis Testing} \label{sec:reduction}

Recall that $f$ is a one-dimensional GP on $[0,1]$ with a stationary kernel $k(x,x')$.  We fix $\Delta > 0$, and think of the GP as being generated by the following procedure:
\begin{enumerate}
    \item Generate a GP $f_0$ with the same kernel on the larger domain $[-\Delta, 1+\Delta]$;
    \item Randomly shift $f_0$ along the $x$-axis by $+\Delta$ or $-\Delta$ with equal probability, to obtain $\ftil$;
    \item Let $f(x) = \ftil(x)$ for $x \in [0,1]$.
\end{enumerate}
Since the kernel is stationary, the shifting does not affect the distribution, so the induced distribution of $f$ is indeed the desired GP on $[0,1]$.

\begin{figure}
    \centering
    \includegraphics[width=0.4\textwidth]{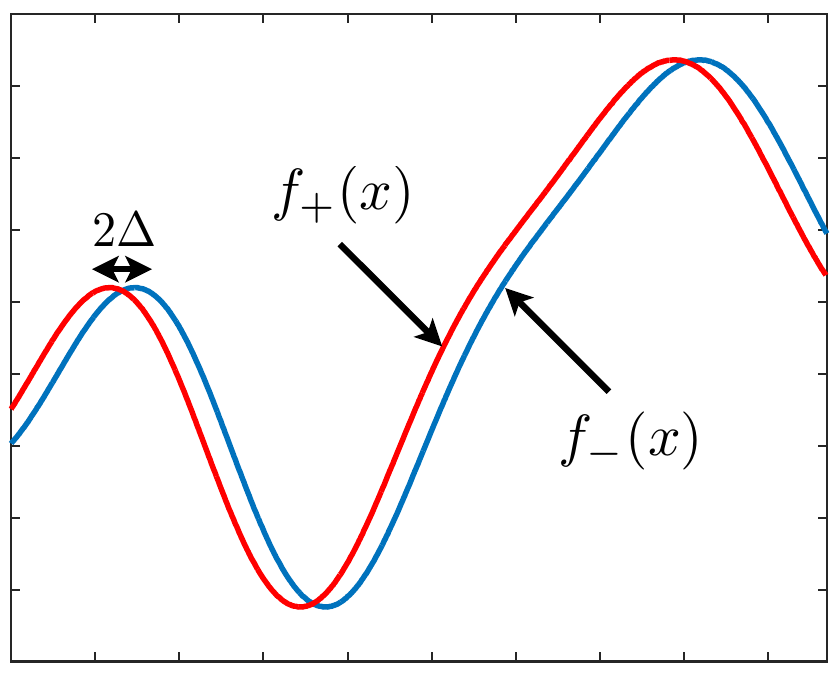}
    \par
    \caption{Examples of functions $f_+$ and $f_-$ considered in the lower bound.  The two are identical up to a small horizontal shift. \label{fig:lb_example}}
\end{figure}

\begin{figure*}
    \centering
    \includegraphics[width=0.65\textwidth]{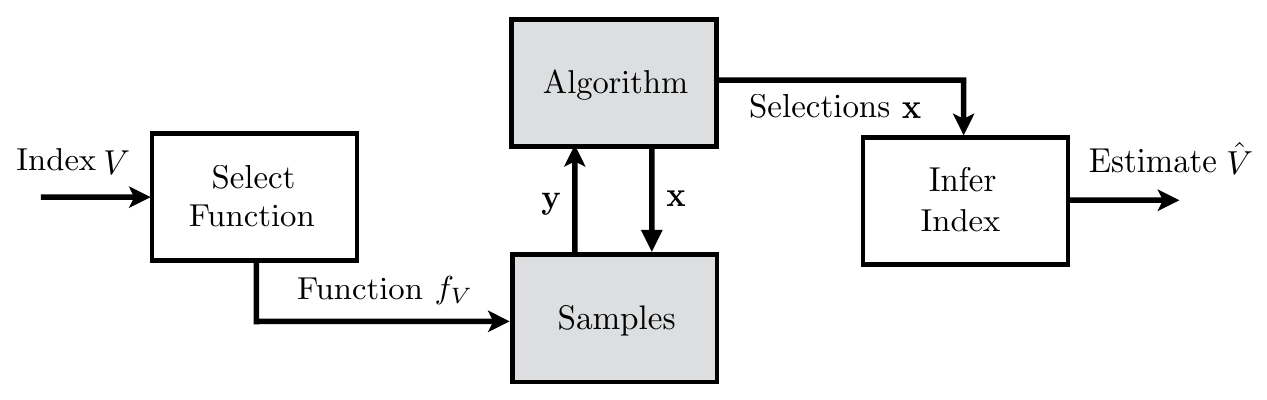}
    \par
    \caption{Illustration of reduction from optimization to binary hypothesis testing.  The gray boxes are considered to be fixed, whereas the white boxes are introduced for the purpose of proving the lower bound. \label{fig:reduction}}
    \vspace*{-2ex}
\end{figure*}

We consider a genie argument in which {\em $f_0$ is revealed to the algorithm}.  Clearly this additional information can only help the algorithm, so any lower bound still remains valid for the original setting.  Stated differently, the algorithm knows that $f$ is either $f_+$ or $f_-$, where
\begin{align}
f_+(x) &= f_0(x+\Delta), \label{eq:f+} \\
f_-(x) &= f_0(x-\Delta). \label{eq:f-}
\end{align}
See Figure \ref{fig:lb_example} for an illustrative example.  

This argument allows us to reduce the BO problem to a {\em binary hypothesis test} with {\em adaptive sampling}, as depicted in Figure \ref{fig:reduction}.  The hypothesis, indexed by $v \in \{-,+\}$, is that the underlying function is $f_v$.  We show that under a suitable choice of $\Delta$, achieving small cumulative regret means that we can construct a decision rule $\hat{V}(\xv)$ such that $\hat{V} = v$ with high probability, i.e., the hypothesis test is successful.  The contrapositive statement is then that {\em if the hypothesis test cannot be successful, we cannot achieve small cumulative regret}, from which it only remains to prove the former.  This idea was used previously for stochastic convex optimization in \cite{Rag11}.

In the remainder of the analysis, we implicitly condition on an arbitrary realization of $f_0$, meaning that all expectations and probabilities are only with respect to the random index $V$ and/or the noise.  We also assume that $f_0$ satisfies the conditions in Assumptions \ref{as:kernel_basic}, \ref{as:kernel_diff}, and \ref{as:endpoints}, which holds with probability at least $1 - \delta_1 - \delta'_2$.   For sufficiently small $\Delta$, the same assumptions are directly inherited by $f_+$ and $f_-$.  We henceforth assume that $\Delta$ is indeed sufficiently small; we will verify that this is the case when we set its value.

We introduce some further notation.  Letting $x_+^*$, $x_-^*$, and $x_0^*$ denote the maximizers of $f_+$, $f_-$ and $f_0$ (which are unique by Assumption \ref{as:kernel_diff}), we see that Assumption \ref{as:endpoints} ensures these are in the interior $(0,1)$, and hence the optimal values coincide: $f_+(x_+^*) = f_-(x_-^*) = f_0(x_0^*) =: f^*$.  To simplify some of the notation, instead of working with these functions directly, we consider the equivalent problem of {\em minimizing the corresponding regret functions}:
\begin{align}
r_+(x) &= f^* - f_+(x), \\
r_-(x) &= f^* - f_-(x).
\end{align}
Indeed, since we assume the algorithm knows $f_0$ and hence also the optimal value $f^*$, it can always choose to transform the samples as $y \to f^* - y$.  In this form, we have the convenient normalization $r_+(x_+^*) = r_-(x_-^*) = 0$.

\subsection{Auxiliary Lemmas}

We first state the following useful properties of $r_+$ and $r_-$.

\begin{lem} \label{lem:r_properties}
    The functions $r_+$ and $r_-$ constructed above satisfy the following for sufficiently small $\Delta$ under the conditions in Assumptions \ref{as:kernel_diff} and \ref{as:endpoints}:
    \begin{enumerate}
        \item We have for all $x \in D$ that
        \begin{equation}
        r_+(x) < \cunder_2 \Delta^2 \implies r_-(x) > \cunder_2 \Delta^2, \label{eq:prop1}
        \end{equation}
        where $\cunder_2$ is defined in Assumption \ref{as:endpoints}.  
        \item There exists a constant $c' > 0$ such that, for all $x \in D$,
        \begin{gather}
        |r_+(x) - r_-(x)| \le c'\big( \Delta |x - x_0^*| + \Delta^2 \big).  \label{eq:prop2}
        \end{gather}
        \item There exists a constant $c'' > 0$ such that, for all $x \in D$, 
        \begin{align}
        r_+(x) &\ge c'' ( (x - x_0^*) + \Delta )^2, \nonumber \\
        r_-(x) &\ge c'' ( (x - x_0^*) - \Delta )^2. \label{eq:quadratic_regret}
        \end{align}
    \end{enumerate}
\end{lem}
\begin{proof}
    See Appendix \ref{sec:pf_r_properties}.
\end{proof}

The first part states that any point can be better than $\cunder_2 \Delta^2$-optimal for at most one of the two functions, the second part shows that the two functions are close for points near $x_0^*$, and the third part shows that the instant regret is lower bounded by a quadratic function.

The first part of the Lemma \ref{lem:r_properties} allows us to bound the cumulative regret using Fano's inequality for binary hypothesis testing with adaptive sampling \cite{Rag11}.  This inequality lower bounds the success probability of such a hypothesis test in terms of a mutual information quantity \cite{Cov01}.  The resulting lower bound on regret is stated in the following; it is worth noting that the consideration of {\em cumulative} regret here provides a distinction from the analogous bound on the instant regret in \cite{Rag11}.

\begin{lem} \label{lem:Fano}
    Under the preceding setup, we have
    \begin{equation}
    \EE[R_T] \ge \cunder_2 T\Delta^2 \cdot H_2^{-1}\big( \log 2 - I(V;\xv,\yv) \big), \label{eq:Fano}
    \end{equation}
    where $V$ is equiprobable on $\{+,-\}$, and $(\xv,\yv)$ are the selected points and samples when the minimization algorithm is applied to $r_V$.  Here $H_2^{-1} \,:\, [0,\log 2] \to \big[0,\frac{1}{2}\big]$ is the functional inverse of the binary entropy function $H_2(\alpha) = \alpha\log\frac{1}{\alpha} + (1-\alpha)\log\frac{1}{1-\alpha}$.
\end{lem}

Since this result is particularly fundamental to our analysis, we provide a proof at the end of this section.

\subsection{Outline of Proof of Theorem \ref{thm:lb}}

Here we provide a high-level outline of the proof of Theorem \ref{thm:lb}; the details are given in Appendix \ref{sec:pf_lb}.

Once the lower bound in Lemma \ref{lem:Fano} is established, the main technical challenge is upper bounding the mutual information.  A useful property called {\em tensorization} (e.g., see \cite{Rag11}) allows us to simplify the mutual information with the vectors $(\xv,\yv)$ to a sum of mutual informations containing only a single pair $(x_t,y_t)$: $I(V;\xv,\yv) \le \sum_{t=1}^T I(V;y_t|x_t)$.  

Each such mutual information term $I(V;y_t|x_t)$ can further be upper bounded by the KL divergence \cite{Cov01} between the conditional output distributions corresponding to $r_+$ and $r_-$, which in turn equals $\frac{(r_+(x) - r_-(x))^2}{2\sigma^2}$ when $x_t = x$.  By substituting the property \eqref{eq:prop2} given in Lemma \ref{lem:r_properties}, we find that $I(V;\xv,\yv)$ is upper bounded by a constant times $\frac{1}{\sigma^2}\big(\Delta^2 \EE\big[ \sum_{t=1}^T |x_t - x_0^*|^2 \big] + T\Delta^4\big)$.  If we can further upper bound $I(V;\xv,\yv)$ by a constant in $(0,\log 2)$, then \eqref{eq:Fano} establishes an $\Omega(T \Delta^2)$ lower bound.

We proceed by considering the cases $\EE[R_T] \ge c'' T \Delta^2$ and $\EE[R_T] < c'' T \Delta^2$ separately, with $c''$ given in \eqref{eq:quadratic_regret}.  The former case will immediately give the lower bound in Theorem \ref{thm:lb} when we set $\Delta$, whereas in the latter case, we can use \eqref{eq:quadratic_regret} to show that $\EE\big[ \sum_{t=1}^T |x_t - x_0^*|^2 \big]$ is upper bounded by a constant times $T \Delta^2$, which means that the desired mutual information upper bound (see the previous paragraph) is attained under a choice of $\Delta$ scaling as $\big( \frac{\sigma^2}{T} \big)^{1/4}$.  Under this choice, the lower bound $\EE[R_T] = \Omega(T \Delta^2)$  evaluates to $\Omega(\sigma\sqrt{T})$, as required.

\subsection{Proof of Lemma \ref{lem:Fano}} \label{sec:pf_Fano}

As mentioned above, the proof of Lemma \ref{lem:Fano} follows along the lines of \cite{Rag11}, which in turn builds on previous works using Fano's inequality to establish minimax lower bounds in statistical estimation problems; see for example \cite{Yu97}.

In the following, we use $R_{T,+} = \sum_{t=1}^T r_+(x_t)$ and $R_{T,-} = \sum_{t=1}^T r_-(x_t)$ to denote the cumulative regret associated with $r_+$ and $r_-$, respectively, and we generically write $R_{T,v}$ to denote one of the two with $v \in \{+,-\}$.

We first use Markov's inequality to write
\begin{equation}
\EE[ R_T ] \ge (1-\alpha)\cunder_2 T \Delta^2 \cdot  \PP[ R_T \ge (1-\alpha)\cunder_2 T \Delta^2 ] \label{eq:ConvExcessDist} 
\end{equation}
for any $\alpha \in (0,1)$.  We proceed by analyzing the probability on the right-hand side.

Recall that $V$ is equiprobable on $\{+,-\}$, and $(\xv,\yv)$ are generated by running the optimization algorithm on $r_V$.  Given the sequence of inputs $\xv$, let $\hat{V}$ be the index $\hat{v} \in \{+,-\}$ with the lower cumulative regret $R_{T,\hat{v}} = \sum_{t=1}^T r_{\hat{v}}(x_t)$.  By Lemma \ref{lem:r_properties}, $R_T$ can be less than $\cunder_2 T \Delta^2$ for at most one of the two functions, and hence, if $R_{T,v} \le (1-\alpha)\cunder_2 T \Delta^2$ then we must have $\hat{V} = v$.  Therefore, 
\begin{equation}
\PP_{v}\big[ R_T \ge (1-\alpha)\cunder_2 T \Delta^2 \big] \ge \PP_v[\hat{V} \ne v], \label{eq:ConvExcessDistProb} 
\end{equation}
where, here and subsequently, $\PP_v$ and $\EE_v$ denote probabilities and expectations when the underlying instant regret function is $r_v$ (i.e., the underlying function that the algorithm seeks to maximize is $f_v$).

Continuing, we can lower bound the probability appearing in \eqref{eq:ConvExcessDist} as follows:
\begin{align}
&\PP[ R_T \ge (1-\alpha)\cunder_2 T \Delta^2 ] \nonumber \\
&\quad= \frac{1}{2} \sum_{v \in \{+,-\} } \PP_{v}\big[ R_T \ge (1-\alpha)\cunder_2 T \Delta^2 \big] \label{eq:ConvexEnd1} \\
&\quad\ge \frac{1}{2} \sum_{v \in \{+,-\} } \PP_{v}[\hat{V} \ne v] \label{eq:ConvexEnd2}  \\
&\quad\ge H_2^{-1}\big( \log 2 - I(V;\xv,\yv) \big), \label{eq:ConvexEnd4} 
\end{align}
where \eqref{eq:ConvexEnd2} follows from \eqref{eq:ConvExcessDistProb}, and \eqref{eq:ConvexEnd4} follows from Fano's inequality for binary hypothesis testing with adaptive sampling (see Eq.~(22) and (24) of \cite{Rag11}).  The proof is completed by combining \eqref{eq:ConvExcessDist} and \eqref{eq:ConvexEnd4}, and recalling that $\alpha$ can be arbitrarily small.

\section{Conclusion and Discussion}

We have established tight scaling laws on the regret for Bayesian optimization in one dimension, showing that the optimal scaling is $\Omega(\sqrt{T})$ and $O(\sqrt{T\log T})$ under mild technical assumptions on the kernel.  Our results highlight some limitations of the widespread upper bounds based on the information gain, as well as providing cases where the noisy Bayesian setting is provably less difficult than its non-Bayesian RKHS counterpart.

An immediate direction for further work is to sharpen the constant factors in the upper and lower bounds, and to establish whether the upper bound is attained by any algorithm that can also provide state-of-the-art performance in practice.  We re-iterate that our algorithm is certainly not suitable for this purpose, as its cumulative regret contains large constant factors, and the algorithm makes use of a variety of specific constants present in the assumptions (though they are merely a function of the kernel).

We expect our techniques to extend to any constant dimension $d \ge 1$; the main ideas from the noiseless upper bound still apply \cite{Fre12}, and in the lower bound we can choose an arbitrary single dimension and introduce a random shift in that direction as per Section \ref{sec:reduction}.  While these extensions may still yield $\sqrt T \,\mathrm{poly}(\mathrm{\log} T)$ regret, the dependence on $d$ would be exponential or worse in the upper bound, but constant in the lower bound, with the latter dependence certainly being suboptimal.  Multi-dimensional lower bounding techniques based on Fano's inequality \cite{Rag11} may improve the latter to $\mathrm{poly}(d)$, but overall, attaining a sharp joint dependence on $T$ and $d$ appears to require different techniques.

\newpage
{\bf Acknowledgments.}
I would like to thank Ilija Bogunovic for his helpful comments and suggestions.  This work was supported by an NUS startup grant. 

\section{Errata} \label{sec:errata}

This section has been added to the arXiv paper to correct two minor mistakes in the published ICML 2018 paper.  I am grateful to Shogo Iwazaki for pointing these out.

\subsection{Correction to the Proof of Theorem \ref{thm:ub}}

The upper bound in Theorem \ref{thm:ub} is stated using a probability bound with respect to $f$ and an expectation with respect to the noise.  However, in the current analysis of the expected regret, the confidence bound in Lemma \ref{lem:conf_bounds} is used, and such a confidence bound holds with respect to both the function and the noise.  This means that Lemma \ref{lem:conf_bounds} is applied in a part of the analysis where we are only meant to be studying the randomness of the noise.  Specifically, the current analysis averages over the event $\Bc$ in \eqref{eq:avg_R_1}--\eqref{eq:avg_R_3} and this event inadvertently includes randomness in $f$.

However, it is stated following \eqref{eq:epoch_length} that, after repeatedly sampling a particular point, ``we performed enough repetitions to attain a variance of at most $\frac{\eta_{(i)}^2}{4\beta_T}$ based on those samples alone'', i.e., any possible variance reduction from querying other points is ignored.  As a result, we may avoid Lemma \ref{lem:conf_bounds} altogether and instead use a simpler Gaussian concentration inequality with respect to the noise alone.\footnote{At this point in the analysis, we can treat $f$ as being fixed, i.e., we have already conditioned on any fixed function satisfying Assumptions \ref{as:kernel_diff} and \ref{as:taylor}.}  Specifically, if $Z \sim N(\mu,\sigma^2)$ and we let $\widehat{\mu}_K$ be the empirical mean of $K$ independent observations of $Z$, then it holds with probability at least $1-\delta_0$ that
\begin{equation}
    |\mu - \widehat{\mu}_K| \le \sigma \sqrt{ \frac{2\log\frac{1}{\delta_0}}{K}}. \label{eq:new_conf}
\end{equation}
We can apply this result (separately) for each repeatedly-sampled point in the algorithm, with $K = K_{(i)}$ as specified in Line 8 of Algorithm \ref{alg:full}.  We need to set $\delta_0$ to be small enough to permit a union bound over $|\mathcal{L}_{(i)}|$ points and all epochs $i=1,2,\dotsc$, and a similar argument to \eqref{eq:beta} reveals that $\delta_0 = \frac{\delta}{2c_1 T^3}$ suffices, thus giving the same confidence width as the original analysis.  In Line 10 of Algorithm \ref{alg:full}, we can use the empirical mean of the $K^{(i)}$ responses (to querying $x'$) in place of the GP posterior mean, in accordance with \eqref{eq:new_conf}, and similarly replace all occurrences of $\mu_{t-1}(\cdot)$ in the analysis.  Apart from this, the analysis is unchanged and the same final result is obtained.

In summary, Theorem \ref{thm:ub} still holds as stated, but the algorithm should use a separate confidence bound for each repeatedly-sampled point based on simple Gaussian concentration, rather than using GP-based confidence bounds.  Note that for the ``purely high-probability'' version discussed after the statement of Theorem \ref{thm:ub}, this change is no longer necessary (i.e., it can be done but doesn't need to be).

\subsection{Correction to the Statement of Theorem \ref{thm:lb}}

Theorem \ref{thm:lb} is stated as holding with probability at least $1-\delta_1-\delta'_2$ with respect to $f$, and as averaging the regret with respect to only the noise.  However, the proof is based on interpreting $f$ as first drawing a function $f_0$ and then shifting it in a random direction $V \in \{+,-\}$, and in the current analysis, the cumulative regret also averages over $V$ (not just the noise).  Thus, the actual result proved is that in which the probability is with respect to $f_0$ and the expectation is with respect to both $V$ and the noise.  Such a result is somewhat unnatural, so here we explain how to instead obtain a slight variant of Theorem \ref{thm:lb} that doesn't change which random variables are involved in the probability and the expectation.

The idea is to decompose the average regret given $f_0$ as
\begin{equation}
    \EE[R_T \,|\, f_0] = \sum_{v \in \{+,-\}} P_V(v) \EE[R_T \,|\, f_0,v], \label{eq:avg_v}
\end{equation}
where $P_V(+) = P_V(-) = \frac{1}{2}$.  In the notation of \eqref{eq:lb}, our analysis shows that $\EE[R_T \,|\, f_0] \ge C'\big( 1+ \sigma \sqrt{T} \big)$ with probability at least $1-\delta_1-\delta'_2$ with respect to $f_0$.  By \eqref{eq:avg_v}, the same lower bound must apply to at least one of the two values of $\EE[R_T \,|\, f_0,v]$.  But $(f_0,v)$ collectively determine $f$, and whichever $v$ value satisfies the desired lower bound, it has probability $\frac{1}{2}$.  Consequently, the desired lower bound $\EE[R_T \,|\, f] \ge C'\big( 1+ \sigma \sqrt{T} \big)$ holds with probability at least $(1-\delta_1-\delta'_2)/2$ with respect to $f$.

In summary, Theorem \ref{thm:lb} holds with the probability halved from $1-\delta_1-\delta'_2$ to $(1-\delta_1-\delta'_2)/2$, and establishing this only requires a small amount of additional reasoning.

\bibliographystyle{icml2018}
\bibliography{refs,../JS_References}

\newpage

\onecolumn

{\Huge \bf \centering Supplementary Material \par}

{\Large \bf \centering Tight Regret Bounds for Bayesian Optimization in One Dimension \\ (Jonathan Scarlett, ICML 2018) \par}

\appendix

\section{Doubling Trick for an Unknown Time Horizon} \label{sec:doubling}

Suppose that we have an algorithm that depends on the time horizon $T'$ and achieves $\EE[R_{T'}] \le C\sqrt{T'\log T'}$ for some $C > 0$.  We show that we can also achieve $\EE[R_T] = O\big( \sqrt{T\log T} \big)$ when $T$ is unknown.

To see this, fix an arbitrary integer $T_0 \in \big[1, \frac{T}{2} \big]$, and repeatedly run the algorithm with fixed time horizons $T_0$, $2T_0$, $4T_0$, etc., until $T$ points have been sampled.  The number of stages is no more than $\ell_{\max} = \lceil \log_2 \frac{T}{T_0} \rceil$.  Moreover, we have
\begin{align}
\EE[R_T] \le \sum_{\ell=1}^{\ell_{\max}} C\sqrt{ 2^{\ell-1}T_0 \log T } % \label{eq:first_line} \\
= C \sqrt{T_0 \log T} \sum_{\ell=0}^{\lceil \log_2 \frac{T}{T_0} \rceil-1} \sqrt{ 2^{\ell} } % \\
\le  C \sqrt{\log T} \cdot 4 \sqrt{T} %, \label{eq:last_line}
\end{align}
where the first inequality uses $\log (2^{\ell-1} T_0) \le \log T$, and the last inequality uses $\sum_{\ell=0}^N 2^{\ell/2} \le 4 \cdot 2^{N/2}$.  % This establishes the desired claim.

\begin{algorithm} 
    \caption{Full description of our algorithm, based on reducing uncertainty in epochs via repeated sampling.} \label{alg:full}
    \begin{algorithmic}[1]
        \REQUIRE Domain $D$, GP prior ($\mu_0$, $k_0$), time horizon $T$, constants $c_0,c_1,c_2,\rho_0$.
        \STATE Set discrete sub-domain $\Lc = \big(\frac{1}{c_1\cdot T}\ZZ \cap [0,1]\big) \cup \{1\}$, confidence parameter $\beta_T = 2\log(2c_1T^3)$, initial target confidence $\eta_{(0)} = c_0$, and initial potential maximizers $M_{(0)}=\Lc$.
        \STATE Initialize time index $t=1$ and epoch number $i=1$.
        \WHILE{less than $T$ samples have been taken}
        \STATE Set $\eta_{(i)} = \frac{1}{2} \eta_{(i-1)}$.
        \STATE Define the interval $$\Ic_{(i)} = \big[ \min\{x \in M_{(i-1)}\}, \max\{x \in M_{(i-1)}\} \big] \cap \Lc,$$ and its width $$w_{(i)} = \max\{x \in M_{(i-1)}\} - \min\{x \in M_{(i-1)}\}.$$
        \STATE Set the Lipschitz constant 
        \begin{equation*}
            L_{(i)} = 
            \begin{cases}
                c_1 & w_{(i)} > \rho_0 \\
                c_1 & w_{(i)} \le \rho_0 \text{ and either $0 \in \Ic_{(i)}$ or $1 \in \Ic_{(i)}$} \\
                c_2w_{(i)} & w_{(i)} \le \rho_0 \text{ and $\Ic_{(i)} \subseteq (0,1)$. }
            \end{cases}
        \end{equation*}
        \STATE Construct a subset $\Lc_{(i)} \subseteq \Ic_{(i)}$ as follows:
        \begin{itemize}
            \item Initialize $\Lc_{(i)} \leftarrow \emptyset$.
            \item Construct $\Lctil_{(i)}$ (not necessarily a subset of $\Ic_{(i)}$ or $\Lc$) containing regularly-spaced points within the interval $\big[ \min\{x \in \Ic_{(i)}\}, \max\{x \in \Ic_{(i)}\}\big]$, with spacing $\frac{\eta_{(i)}}{2L_{(i)}}$.
            \item For each $x \in \Lctil_{(i)}$, add its two nearest points in $\Ic_{(i)}$ to $\Lc_{(i)}$.
        \end{itemize}
        \STATE Sample each point in $\Lc_{(i)}$ repeatedly $K_{(i)}$ times, where
        \begin{equation*}
            K_{(i)} = \Big\lceil \frac{4\sigma^2 \beta_T}{\eta_{(i)}^2} \Big\rceil.
        \end{equation*}
        For each sample taken, increment $t \leftarrow t+1$, and terminate if $t > T$.
        \STATE Update the posterior distribution $(\mu_{t-1},\sigma_{t-1})$ according to \eqref{eq:mu_update}--\eqref{eq:sigma_update}, with $\xv_{t-1} = [x_1,\dotsc,x_{t-1}]^T$ and  $\yv_{t-1} = [y_1,\dotsc,y_{t-1}]^T$ respectively containing all the selected points and noisy samples so far.
        \STATE For each $x \in \Ic_{(i)}$, set 
        \begin{align*}
            \UCB_t( x ) = \mu_{t-1}(x') + \eta_{(i)}, \quad \LCB_{t}( x ) = \mu_{t-1}(x') - \eta_{(i)},
        \end{align*}
        where $x' = \argmin_{x' \in \Lc_{(i)}} |x - x'|$.
        \STATE Update the set of potential maximizers:
        \begin{equation*}
        \hspace*{-4ex} M_{(i)} = \Big\{ x \in M_{(i-1)} \,:\, \UCB_t(x) \ge \max_{x' \in M_{(i-1)}} \LCB_t(x') \Big\}.
        \end{equation*}
        % \STATEx $\quad$~ where $\UCB_t$ and $\LCB_t$ are suitably-chosen upper and lower confidence bounds.
        \STATE Increment $i$.
        \ENDWHILE
    \end{algorithmic}
\end{algorithm}

% \section{Proof of Lemma \ref{lem:elimination}} \label{sec:pf_elimination}

\vspace*{-2ex}
\section{Proof of Theorem \ref{thm:ub} (Upper Bound)} \label{sec:pf_ub}

We continue from the auxiliary results given in Section \ref{sec:ub}, proceeding in several steps.  Algorithm \ref{alg:full} gives a full description of the algorithm; the reader is encouraged to refer to this throughout the proof, rather than trying to understand all the steps therein immediately.   Note that the constants $c_0$, $c_1$, $c_2$, and $\rho_0$ used in the algorithm come from Assumptions \ref{as:kernel_diff} and \ref{as:taylor}.

% The reader is encouraged to refer to the full description of the algorithm in Algorithm \ref{alg:full} throughout the proof.

{\bf Reduction to a finite domain.} Our algorithm only samples $f$ within a finite set $\Lc \subseteq D$ of pre-defined points.  We choose these points to be regularly spaced, and close enough to ensure that the highest function value is within $\frac{1}{T}$ of the maximum $f(x^*)$.  Under condition \eqref{eq:c0} in Assumption \ref{as:kernel_diff} (which implies that $f$ is $c_1$-Lipschitz continuous), it suffices to choose
\begin{equation}
\Lc = \bigg(\frac{1}{c_1\cdot T}\ZZ \cap [0,1]\bigg) \cup \{1\}, \label{eq:setL}
\end{equation}
where $\ZZ$ denotes the integers.  Here we add $x=1$ to $\Lc$ because it will be notationally convenient to ensure that the endpoints $\{0,1\}$ are both included in the set.  Note that $\Lc$ satisfies $|\Lc| \le c_1 T + 1$, which we crudely upper bound by $|\Lc| \le 2c_1 T$.

Since $\max_{x \in \Lc} f(x) \ge \max_{x \in D} f(x) - \frac{1}{T}$, the cumulative regret $R_{T}^{(\Lc)}$ with respect to the best point in $\Lc$ is such that
\begin{equation}
R_T \le R_T^{(\Lc)} + 1. \label{eq:R_vs_RL}
\end{equation}
Hence, it suffices to bound $R_T^{(\Lc)}$ instead of $R_T$.  For convenience, we henceforth let $x_{\Lc}^*$ denote an arbitrary input that achieves $\max_{x \in \Lc} f(x)$, and we define the {\em instant regret} as
\begin{equation}
r(x) = f(x^*) - f(x), \quad r_t = r(x_t) = f(x^*) - f(x_t), \quad r_t^{(\Lc)} = f(x^*_{\Lc}) - f(x_t). \label{eq:simple_regret}
\end{equation}

{\bf Conditioning on high-probability events.} By assumption, the events in Assumptions \ref{as:kernel_diff} and \ref{as:taylor} simultaneously hold with probability at least $1-\delta_1-\delta_2$.  Moreover, by setting $\delta = \frac{1}{T}$ in Lemma \ref{lem:conf_bounds} and letting $\Lc$ be as in \eqref{eq:setL} with $|\Lc| \le 2c_1 T$, we deduce that \eqref{eq:conf_bounds} holds with probability at least $1 - \frac{1}{T}$ when\footnote{{\bf Correction:} Lemma \ref{lem:conf_bounds} is no longer used in a corrected version of the proof of Theorem \ref{thm:ub}; see Section \ref{sec:errata} for details.} 
\begin{equation}
\beta_T = 2\log\big(2 c_1 T^3\big). \label{eq:beta}
\end{equation}
Denoting the intersection of all events in Assumptions \ref{as:kernel_diff} and \ref{as:taylor} by $\Ac$, and the event in Lemma \ref{lem:conf_bounds} by $\Bc$, we can write the average regret given $\Ac$ as follows:
\begin{align}
\EE[R_T | \Ac] 
&= \EE[R_T | \Ac, \Bc] \cdot \PP[\Bc | \Ac] + \EE[R_T | \Ac, \Bc^c] \cdot \PP[\Bc^c | \Ac] \label{eq:avg_R_1} \\
&\le \EE[R_T | \Ac, \Bc] + \EE[R_T | \Ac, \Bc^c] \frac{1}{T(1-\delta_1-\delta_2)} \label{eq:avg_R_2} \\
&\le \EE[R_T | \Ac, \Bc] + \frac{2c_0}{1-\delta_1-\delta_2},  \label{eq:avg_R_3}
\end{align}
where \eqref{eq:avg_R_2} follows since $\PP[\Bc | \Ac] \le 1$ and $\PP[\Bc^c | \Ac] \le \frac{\PP[\Bc^c]}{\PP[\Ac]} \le \frac{1}{T(1-\delta_1-\delta_2)}$, and \eqref{eq:avg_R_3} follows since condition \eqref{eq:c0} in Assumption \ref{as:kernel_diff} ensures that $R_T \le T \cdot 2c_0$.  By \eqref{eq:avg_R_3}, in order to prove Theorem \ref{thm:ub}, it suffices to show that $R_T = O(\sqrt{T\log T})$ whenever the conditions of Assumptions \ref{as:kernel_diff}--\ref{as:taylor} and Lemma \ref{lem:conf_bounds} hold true.  We henceforth condition on this being the case.

% In the following, we initially assume that we are in the second case of Assumption \ref{as:taylor} (i.e., $x^* \in (0,1)$); we defer the first case therein (i.e., $x^* =0$ or $x^* = 1$) until later.

{\bf Sampling mechanism.} Recall that $\eta_{(i)}$ represents the target confidence to attain by the end of the $i$-th epoch, and each such value is half of the previous value.  For this interpretation to be valid, $\eta_{(0)}$ should be sufficient large so that the entire function is {\em a priori} known up to confidence $\eta_{(0)}$; by \eqref{eq:c0} in Assumption \ref{as:kernel_diff}, the choice $\eta_{(0)} = c_0$ certainly suffices for this purpose.

In the $i$-th epoch, we repeatedly sample a {\em sufficiently fine} subset of $\Lc$ {\em sufficiently many times} to attain an overall confidence of $\eta_{(i)}$ within $M_{(i-1)}$ (with $M_{(0)} = \Lc$).  Specifically:
\begin{itemize}
    \item We sample each point $K_{(i)}$ times and average the resulting observations, yielding an {\em effective noise variance} of $\frac{\sigma^2}{K_{(i)}}$, and we choose $K_{(i)}$ large enough so that $\frac{\sigma^2}{K_{(i)}} \le \frac{\eta^2_{(i)}}{4\beta_T}$.  Hence, $K_{(i)} = \lceil \frac{4\sigma^2 \beta_T}{\eta_{(i)}^2} \rceil$ is sufficient.
    \item To design $\Lc_{(i)} \subseteq \Lc$, we consider the interval
    \begin{equation}
    \Ic_{(i)} = \big[ \min\{x \in M_{(i-1)}\}, \max\{x \in M_{(i-1)}\} \big] \cap \Lc, \label{eq:setI}
    \end{equation}
    which is the smallest interval (intersected with $\Lc$) containing $M_{(i-1)}$.  We select a Lipschitz constant $L_{(i)}$ (to be specified later) such that $f$ is $L_{(i)}$-Lipschitz within $\Ic_{(i)}$, and then we choose $\Lc_{(i)} \subseteq \Ic_{(i)}$ to ensure the following:
    \begin{equation}
    \text{Each $x \in \Ic_{(i)}$ is within a distance $\frac{\eta_{(i)}}{2L_{(i)}}$ of the nearest $x' \in \Lc_{(i)}$}. \label{eq:cond}
    \end{equation}
    If we were sampling at arbitrary locations, it would suffice to choose $\big\lceil\frac{2w_{(i)} L_{(i)} }{ \eta_{(i)} }\big\rceil$ equally-spaced points, where
    \begin{align}
    w_{(i)} = \max\{x \in M_{(i-1)}\} - \min\{x \in M_{(i-1)}\} \label{eq:wi}
    \end{align}
    is the width of the interval.  With the restriction of sampling within the fine discretization $\Lc$, we can simply ``round'' to the two nearest points,\footnote{To give a concrete example, suppose that $\Lc = \{0,0.01,\dotsc,0.99,1\}$, and that we seek a set of points such that each $x \in \Lc$ is within a distance $\frac{1}{3}$ of the nearest one.  Without constraints, the points $\big\{ \frac{1}{3}, \frac{2}{3} \big\}$ would suffice, but after rounding these to $\{0.33, 0.66\}$, the point $x = 1$ is at a distance $0.34 > \frac{1}{3}$.  However, doubling up and constructing the set $\{0.33, 0.34, 0.66, 0.67\}$ clearly suffices.  } yielding a suitable set $\Lc_{(i)} \subseteq \Ic_{(i)}$ of cardinality at most $2\big\lceil\frac{2 w_{(i)} L_{(i)} }{ \eta_{(i)} }\big\rceil$
\end{itemize}
Combining these, the total number of samples $T_{(i)}$ is given by
\begin{align}
T_{(i)} &= K_{(i)} \cdot |\Lc_{(i)}| \\
& \le 2 \cdot \Big\lceil \frac{4 \sigma^2 \beta_T}{\eta_{(i)}^2} \Big\rceil \cdot \Big\lceil  \frac{2 w_{(i)} L_{(i)} }{ \eta_{(i)} }\Big\rceil . \label{eq:epoch_length}
\end{align}
At the points that were sampled, we performed enough repetitions to attain a variance of at most $\frac{\eta_{(i)}^2}{4\beta_T}$ based on those samples alone.  The information from any earlier samples only reduces the variance further, so the overall posterior variance\footnote{We consider $(\mu_{t-1},\sigma_{t-1})$ instead of $(\mu_{t},\sigma_{t})$ because when the time index is $t$, we have only selected $t-1$ points. } $\sigma_{t-1}^2(x)$ also yields $\beta_T^{1/2} \sigma_{t-1}(x) \le \frac{\eta_{(i)}}{2}$.  Hence, Lemma \ref{lem:conf_bounds} ensures that at these sampled points, we can set
\begin{align}
\UCBtil_t( x ) = \mu_{t-1}(x) + \frac{\eta_{(i)}}{2}, \quad \LCBtil_t( x ) = \mu_{t-1}(x) - \frac{\eta_{(i)}}{2}. \label{eq:ucb_lcb_1}
\end{align}
For the points in $M_{(i-1)}$ that we didn't sample, we note that the following confidence bounds are valid as long as $f$ is $L_{(i)}$-Lipschitz continuous within $\Ic_{(i)}$:
\begin{align}
\UCBtil_t( x ) &= \mu_{t-1}(x') + \frac{\eta_{(i)}}{2} + L_{(i)}|x - x'|, \label{eq:ucb_2} \\ 
\LCBtil_t( x ) &= \mu_{t-1}(x') - \frac{\eta_{(i)}}{2} - L_{(i)}|x - x'|, \label{eq:lcb_2} 
\end{align}
where $x' = \argmin_{x' \in \Lc_{(i)}} |x - x'|$ is the closest sampled point to $x$.  If $x$ is itself in $\Lc_{(i)}$, these expressions reduce to \eqref{eq:ucb_lcb_1}.

Now, since we have ensured the condition \eqref{eq:cond}, we find that we can weaken \eqref{eq:ucb_2}--\eqref{eq:lcb_2} to
\begin{align}
\UCB_t( x ) = \mu_{t-1}(x') + \eta_{(i)}, \quad \LCB_t( x ) = \mu_{t-1}(x') - \eta_{(i)}. \label{eq:ucb_lcb_3}
\end{align}
That is, as long as the Lipschitz constant $L_{(i)}$ is valid, we have $\eta_{(i)}$-confidence at the end of the $i$-th epoch.  As a result, by Lemma \ref{lem:elimination}, the updated set of potential maximizers
\begin{equation}
M_{(i)} = \bigg\{ x \in M_{(i-1)} \,:\, \UCB_t( x ) \ge \max_{x' \in \Lc} \LCB_t( x ) \bigg\},
\end{equation}
with $t$ being the ending time of the epoch, must only contain points within $\Lc$ whose function value is within $4\eta_{(i)}$ of $f(x_{\Lc}^*)$.  Below, we will choose $L_{(i)}$ differently in different epochs, while still ensuring the required Lipschitz condition is valid.

{\bf Analysis of early epochs.} Recall the following:
\begin{itemize}
    \item By Assumption \ref{as:kernel_basic}, the constant $\epsilon$ lower bounds the separation between $f(x^*)$ and the function value at the second highest local maximum (if any). 
    \item By Assumption \ref{as:taylor}, we either have $x^*$ at an endpoint and the locally linear behavior \eqref{eq:linear_opt}, or we have $x^* \in (\rho_0,1-\rho_0)$ and the locally quadratic behavior \eqref{eq:Taylor_opt}.
\end{itemize}
In the epochs for which $w^{(i)} > \rho_0$, we choose $L_{(i)} = c_1$ ({\em cf.}, \eqref{eq:c0}), which is clearly a valid Lipschitz constant.  We claim that {\em after a finite number of epochs}, all points $x \in M_{(i)}$ satisfy $f(x) > f(x^*) - \epsilon$ and $|x - x^*| \le \frac{\rho_0}{2}$, and therefore, $w_{(i)}$ ceases to be greater than $\rho_0$.  We henceforth distinguish between the two cases using the terminology {\em early epochs} and {\em late epochs}.

To see that the preceding claim is true, we consider the two cases of Assumption \ref{as:taylor}:
\begin{itemize}
    \item In the first case, all points satisfying $|x - x^*| > \rho_0$ are at least $\min\{\cunder_1 \rho_0, \epsilon\}$-suboptimal by the locally linear behavior \eqref{eq:linear_opt} and the $\epsilon$ gap \eqref{eq:eta};
    \item In the second case, all points satisfying $|x - x^*| > \rho_0$ are at least $\min\{\cunder_2\rho_0^2,\epsilon\}$-suboptimal by the locally quadratic behavior \eqref{eq:linear_opt} and the $\epsilon$ gap \eqref{eq:eta}.
\end{itemize}
Hence, in either case, all points satisfying $|x - x^*| > \rho_0$ are at least $\epsilon'$-suboptimal, where $\epsilon' = \min\{\cunder_1 \rho_0, \cunder_2\rho_0^2,\epsilon\}$.  As a result, to establish the desired claim, we only need to show that $M_{(i)}$ contains no points with instant regret $r(x) \ge \epsilon'$.

Since $f(x_{\Lc}^*) \ge f(x^*) - \frac{1}{T}$ (as stated following \eqref{eq:setL}), we find that as long as $T > \frac{2}{\epsilon'}$,\footnote{It is safe to assume that $T$ is sufficiently large, since the smaller values of $T$ can be handled by increasing $C$ in the theorem statement.} it suffices that $M_{(i)}$ only contains points such that $r_t^{(\Lc)}(x) \le \frac{\epsilon'}{2}$.  By Lemma \ref{lem:elimination}, this happens as soon as $\eta_{(i)} < \frac{\epsilon'}{8}$.  Since $\epsilon'$ is constant and we halve $\eta_{(i)}$ at the end of each epoch, it must be that only a finite number of epochs $i_{\max,1}$ pass before this occurs, with $i_{\max,1}$ depending only on $\eta_{(0)}$ and $\epsilon'$.

For these early epochs, we simply upper bound $w_{(i)}$ in \eqref{eq:epoch_length} by one, meaning their overall cumulative time $\Tearly$ satisfies
\begin{equation}
\Tearly \le \sum_{i=1}^{i_{\max,1}} T_{(i)} \le 2 i_{\max,1} \Big\lceil \frac{256 \sigma^2 \beta_T}{(\epsilon')^2} \Big\rceil \cdot \Big\lceil \frac{16 c_1 }{ \epsilon' } \Big\rceil, \label{eq:Tearly}
\end{equation}
where we have used the fact that $\eta_{(i)} \ge \frac{\epsilon'}{8}$ and $L_{(i)} = c_1$ in these epochs. 

{\bf Analysis of late epochs.} Recall that we consider ourselves in a late epoch as soon as $w_{(i)} \le \rho_0$.  This condition implies that all points in $M_{(i-1)}$ are within a distance $\rho_0$ of $x^*$,\footnote{Since we condition on the confidence bounds in Lemma \ref{lem:conf_bounds} being valid, only points that are truly suboptimal are ever ruled out.} yielding the locally linear behavior \eqref{eq:linear_opt} if $x^*$ is an endpoint, and the locally quadratic behavior \eqref{eq:Taylor_opt} otherwise.  Moreover, Assumption \ref{as:taylor} assumes $x^* \in (\rho_0,1-\rho_0)$ in the latter case, and as a result, the algorithm can identify which case has occurred: If $\Ic_{(i)}$ contains an endpoint, then we are in the first case, whereas if $\Ic_{(i)} \subseteq (0,1)$, then we are in the second case.  

Accordingly, the algorithm can choose the Lipschitz constant $L_{(i)}$ differently in the two cases.  In the first case, we simply continue to use the global choice $L_{(i)} = c_1$ from \eqref{eq:c0}.  In the second case, we observe that $f'(x^*) = 0$, and recall from \eqref{eq:c0} that $f'$ is $c_2$-Lipschitz continuous.  Since the width of the interval of interest $\Ic_{(i)}$ is $w_{(i)}$, we conclude that $|f'(x)| \le c_2 w_{(i)}$ within $\Ic_{(i)}$, and accordingly, we can set 
\begin{equation}
L_{(i)} = c_2 w_{(i)}. \label{eq:L_choice}
\end{equation}
We initially focus on this second case (which is the more interesting of the two), and later return to the first case.

Recall that within the $i$-th epoch, all points with $f(x) < f(x^*_{\Lc}) - 4\eta_{(i-1)}$ have already been removed from the potential maximizers ({\em cf.}, Lemma \ref{lem:elimination}).  This implies that the points sampled incur instant regret at most
\begin{equation}
r_t^{(\Lc)} \le 4\eta_{(i-1)}, \label{eq:r_later}
\end{equation}
and hence, since we have established that $f(x_{\Lc}^*) \ge f(x^*) - \frac{1}{T}$,
\begin{equation}
r_t \le 4\eta_{(i-1)} + \frac{1}{T}. \label{eq:r_later_2}
\end{equation}
From this fact and the locally quadratic behavior \eqref{eq:Taylor_opt}, we deduce that the width $w_{(i)}$ defined in \eqref{eq:wi} satisfies $w_{(i)} \le \sqrt{\frac{4\eta_{(i-1)} + \frac{1}{T}}{\cunder_2}} = \sqrt{\frac{8\eta_{(i)} + \frac{1}{T}}{\cunder_2}}$ (since $\eta_{(i-1)} = 2\eta_{(i)}$), from which \eqref{eq:epoch_length} and \eqref{eq:L_choice} yield
\begin{align}
T_{(i)} \le 2 \Big\lceil \frac{4 \sigma^2 \beta_T}{\eta_{(i)}^2} \Big\rceil \cdot \Big\lceil \frac{2 c_2}{ \cunder_2 }\cdot \Big( 8 + \frac{1}{T\eta_{(i)}} \Big) \Big\rceil. \label{eq:epoch_length_later}
\end{align} 
Grouping all the constants together and writing $\lceil z \rceil \le 1+z$, we can simplify this to
\begin{equation}
T_{(i)} \le c' \bigg( 1 + \frac{1}{T\eta_{(i)}} + \frac{ \sigma^2 \beta_T }{\eta_{(i)}^2} + \frac{ \sigma^2 \beta_T }{T \eta_{(i)}^3} \bigg) \label{eq:Ti_bound_later}
\end{equation}
for suitably-chosen $c' > 0$.

{\bf Bounding the cumulative regret.} In the early epochs, we crudely upper bound the regret at each time instant by $2c_0$ ({\em cf.}, \eqref{eq:c0}).  Hence, since the total cumulative time of these epochs satisfies \eqref{eq:Tearly} for bounded $i_{\max,1}$, and $\beta_T = O( \log T )$ as per \eqref{eq:beta}, the corresponding total cumulative regret $\Rearly^{(\Lc)}$ is upper bounded by 
\begin{equation}
\Rearly^{(\Lc)} \le c'' (1 + \sigma^2 \log T) \label{eq:R1}
\end{equation} 
for some $c'' > 0$.

For the late epochs, we make use of the instant regret bound in \eqref{eq:r_later}, depending on the epoch index $i$.  Since this upper bound is decreasing in $i$, and the epoch lengths satisfy \eqref{eq:Ti_bound_later}, we can upper bound $R_T^{(\Lc)}$ by considering the hypothetical case that the epoch lengths are {\em exactly} the right-hand side of \eqref{eq:Ti_bound_later}, and the instant regret incurred at time $t$ is {\em exactly} $r_t^{(\Lc)} = 4\eta_{(i-1)}$.

In this situation, we can easily upper bound the total number of epochs: The last epoch must certainly be no larger than $i_{\max,2}$, defined to be the smallest $i$ such that the term $c'\frac{ \sigma^2 \beta_T }{\eta_{(i)}^2}$ on the right-hand side of \eqref{eq:Ti_bound_later} is $T$ or higher.  Substituting $\eta_{(i)} = \frac{\eta_{(0)}}{2^{i}}$ and re-arranging, we conclude that
\begin{equation}
    i_{\max,2} \le \log_4 \frac{T \eta_{(0)}^2}{ c' \sigma^2 \beta_T } = \log_2 \sqrt{\frac{T \eta_{(0)}^2}{ c' \sigma^2 \beta_T }}. \label{eq:imax_bound}
\end{equation}
For technical reasons, here and subsequently we can assume without loss of generality that $\sigma \le \kappa \sqrt{\frac{T}{\log T}}$ for arbitrarily small $\kappa > 0$ and sufficiently large $T$; otherwise, Theorem \ref{thm:ub} states the trivial bound $\EE[R_T] \le CT$.  Since $\beta_T = \Theta(\log T)$, this technical condition means the right-hand side of \eqref{eq:imax_bound} exceeds one.

Continuing, the total cumulative regret $\Rlate^{(\Lc)}$ from the late epochs is upper bounded as follows:
\begin{align}
\Rlate^{(\Lc)} 
&\le \sum_{i=1}^{i_{\max,2}} 4\eta_{(i-1)} T^{(i)} \label{eq:r2_bound1} \\
&\le 4c' \sum_{i=1}^{i_{\max,2}} \eta_{(i-1)} + 8c'\bigg( \sum_{i=1}^{i_{\max,2}} 1\bigg) +  8c' \sigma^2 \beta_T \sum_{i=1}^{i_{\max,2}}  \frac{1}{\eta_{(i)}} + \frac{8c' \sigma^2 \beta_T}{T} \sum_{i=1}^{i_{\max,2}}  \frac{1}{\eta_{(i)}^2} \label{eq:r2_bound2} \\
&\le 4c' i_{\max,2}(\eta_{(0)} + 2) +  8c' \sigma^2 \beta_T \sum_{i=1}^{i_{\max,2}}  \frac{1}{\eta_{(i)}} + \frac{8c' \sigma^2 \beta_T}{T} \sum_{i=1}^{i_{\max,2}}  \frac{1}{\eta_{(i)}^2} \label{eq:r2_bound2a} \\
&\le 4c' i_{\max,2}(\eta_{(0)} + 2) + \frac{8c' \sigma^2 \beta_T}{\eta_{(0)}} \sum_{i=1}^{i_{\max,2}}  2^i +  \frac{8c' \sigma^2 \beta_T}{T\eta_{(0)}^2} \sum_{i=1}^{i_{\max,2}}  4^i \label{eq:r2_bound3} \\
&\le 4c' i_{\max,2}(\eta_{(0)} + 2) + \frac{16c' \sigma^2 \beta_T}{\eta_{(0)}} 2^{i_{\max,2}} + \frac{16 c' \sigma^2 \beta_T}{T\eta_{(0)}^2} 4^{i_{\max,2}} \label{eq:r2_bound4} \\
&\le 4c' (\eta_{(0)} + 2) \log_4 \frac{T \eta_{(0)}^2}{ c' \sigma^2 \beta_T } + 16 \sqrt{ c' \sigma^2 \beta_T T} + 16,\label{eq:r2_bound5} 
\end{align}
where \eqref{eq:r2_bound2} follows from \eqref{eq:Ti_bound_later} and the fact that $\eta_{(i-1)} = 2\eta_{(i)}$, \eqref{eq:r2_bound2a} follows since $\eta_{(i-1)} \le \eta_{(0)}$, \eqref{eq:r2_bound3} follows since $\eta_{(i)} = \frac{\eta_{(0)}}{2^i}$, \eqref{eq:r2_bound4} follows since $\sum_{i=1}^N 2^i \le 2\cdot 2^N$ and $\sum_{i=1}^N 4^i \le 2\cdot 4^N$, and \eqref{eq:r2_bound5} follows by substituting the upper bound on $i_{\max,2}$ from \eqref{eq:imax_bound}.  Using the fact that $\beta_T = O(\log T)$, and recalling that $\eta_{(0)} = c_0$ is constant, we simplify \eqref{eq:r2_bound5} to
\begin{equation}
\Rlate^{(\Lc)} \le c^{\dagger} \Big( 1 + \sigma \sqrt{T \log T} \Big) \label{eq:r2_bound6}
\end{equation}
for some $c^{\dagger} > 0$.  Note that we can safely drop the $O\big(\log \frac{T}{\sigma^2 \beta_T}\big) = O\big(\log \frac{T}{\sigma^2 \log T}\big)$ term in \eqref{eq:r2_bound5} due to the assumption $\sigma^2 \ge \frac{c_{\sigma}}{T^{1-\zeta}}$ in Theorem \ref{thm:ub}.

{\bf Handling the first case in Assumption \ref{as:taylor}.} From \eqref{eq:L_choice} onwards, we focused only on the second case of Assumption \ref{as:taylor}.  In the first case, we have a worse Lipschitz constant $L_{(i)} = c_1$, but the width also shrinks faster: By the locally linear behavior \eqref{eq:linear_opt}, achieving $\eta_{(i)}$-confidence not only brings the interval width $w_{(i)}$ down to at most $O(\sqrt{\eta_{(i)}})$, but also further down to $O(\eta_{(i)})$.  Hence, we lose a factor of $\sqrt{\eta_{(i)}}$ in the Lipschitz constant, but we gain a factor of $\sqrt{\eta_{(i)}}$ in the upper bound on $w_{(i)}$.  Since the number of points sampled in \eqref{eq:epoch_length} contains the product of the two, the final result remains unchanged, i.e., we still have \eqref{eq:r2_bound6}, possibly with a different constant $c^{\dagger}$.

{\bf Completion of the proof.} Combining \eqref{eq:R_vs_RL}, \eqref{eq:avg_R_3}, \eqref{eq:R1} and \eqref{eq:r2_bound6}, we obtain 
\begin{equation}
    \EE[R_T] \le C^{\dagger} \big( 1 + \sigma^2 \log T + \sigma \sqrt{T \log T}\big) \label{eq:ub2}
\end{equation}
for some constant $C^{\dagger}$.   As stated following \eqref{eq:imax_bound}, we can assume without loss of generality that $\sigma \le  O\big(\sqrt{\frac{T}{\log T}}\big)$, which means that the third term of \eqref{eq:ub2} dominates the second, and the proof is compete.

%In the case that $\sigma \le \sqrt{\frac{T}{\log T}}$, the third term dominates the second, and we obtain the desired bound in Theorem \ref{thm:ub}.  On the other hand, if $\sigma > \sqrt{\frac{T}{\log T}}$, then Theorem \ref{thm:ub} is trivial anyway, as it states that $\EE[R_T] \le C T$, and achieving $R_T \le 2c_0 T$ is trivial due to \eqref{eq:c0}.
 % Note that we have used the assumption $\sigma^2 = O(1)$ to simplify $\sigma^2 = O(\sigma)$, and we have also simplified $(\log T)^{5/8}$ by upper bounding it by $\log T$.  These final simplifications are done only to neaten the statement of the result.

\section{Proof of Lemma \ref{lem:r_properties}} \label{sec:pf_r_properties}

For the first part, we consider $\Delta$ sufficiently small so that $\cunder_2 \Delta^2 < \epsilon$, for $\epsilon$ given in Assumption \ref{as:kernel_diff} and $\cunder_2$ in Assumption \ref{as:endpoints}.  Since all local maxima are at least $\epsilon$-suboptimal, achieving $r_+(x) < \cunder_2 \Delta^2$ requires that $x$ lies within a small interval around $x_+^*$.  Moreover, the locally quadratic behavior \eqref{eq:Taylor_opt2} in Assumption \ref{as:endpoints} yields $r_+(x) \ge \cunder_2 (x-x_+^*)^2$ within this interval when $\Delta$ is sufficiently small.  Combining this with $r_+(x) < \cunder_2 \Delta^2$ gives $|x - x_+^*| < \Delta$, and since $|x^*_+ - x^*_-| = 2\Delta$, the triangle inequality yields $|x - x_-^*| > \Delta$.  Again using \eqref{eq:Taylor_opt2}, we conclude that $r_-(x) > \cunder_2 \Delta^2$, as required.

For the second part, we recall from \eqref{eq:f+}--\eqref{eq:f-} that $r_+(x) = r_0(x+\Delta)$ and $r_-(x) = r_0(x-\Delta)$, where $r_0(x) = f_0(x_0^*) - f_0(x)$.  Again assuming $\Delta$ is sufficiently small (i.e., less than $\rho_0$), we can apply the general Taylor expansion according to \eqref{eq:Taylor_general} to obtain
\begin{gather}
|r_+(x) - r_0(x)| \le \Delta |r'_0(x)| + c_{2,\max} \Delta^2, \label{eq:abs1} \\
|r_-(x) - r_0(x)| \le \Delta |r'_0(x)| + c_{2,\max} \Delta^2, \label{eq:abs2} 
\end{gather}
where $c_{2,\max} = \max\{ |\cunder'_2|, |\cbar'_2| \}$.  Since $r'_0(x)$ is $c_2$-Lipschitz continuous (see \eqref{eq:c0}) and equals zero at $x_0^*$, we must have $|r'_0(x)| \le c_2 |x - x_0^*|$.  Hence, and using the triangle inequality along with \eqref{eq:abs1}--\eqref{eq:abs2}, we have
\begin{equation}
|r_+(x) - r_-(x)| \le 2 \Delta c_2 |x - x_0^*| + 2c_{2,\max} \Delta^2,  
\end{equation}
which proves \eqref{eq:prop2}.

For the third part, we note that since $x_+^* = x_0^* - \Delta$ and $x_-^* = x_0^* + \Delta$, the conditions in \eqref{eq:quadratic_regret} can be written as
\begin{align}
    r_+(x) \ge c'' ( x - x_+^* )^2, \qquad r_-(x) \ge c'' ( x - x_-^* )^2.
\end{align}
Using the locally quadratic behavior in \eqref{eq:Taylor_opt2}, we deduce that \eqref{eq:quadratic_regret} holds for all $x$ within distance $\rho_0$ of the respective function optimizer.  On the other hand, if the distance from the optimizer is more than $\rho_0$, then a combination of \eqref{eq:eta} and \eqref{eq:Taylor_opt2} reveals that $r(x)$ is bounded away from zero.  Since the quadratic terms in \eqref{eq:quadratic_regret} are also bounded from above due to the fact that $x \in [0,1]$, we conclude that \eqref{eq:quadratic_regret} holds for sufficiently small $c''$.

\section{Proof of Theorem \ref{thm:lb} (Lower Bound)} \label{sec:pf_lb}

We continue from the reduction to binary hypothesis testing and auxiliary results given in Section \ref{sec:lb}.  These results hold for an arbitrary given (deterministic) BO algorithm, which in general is simply a sequence of mappings that return the next point $x_t$ based on the previous samples $y_1,\dotsc,y_{t-1}$.  Recall also that we implicitly condition on an arbitrary realization of $f_0$ satisfying the events in Assumptions \ref{as:kernel_diff} and \ref{as:endpoints}, meaning that all expectations and probabilities are only with respect to the random index $V \in \{+,-\}$ and/or the noise.  We proceed in two main steps.

{\bf Bounding the mutual information.} To bound the mutual information term $I(V;\xv,\yv)$ appearing in \eqref{eq:Fano}, we first apply the following tensorization bound for adaptive sampling, which is based on the chain rule for mutual information (e.g., see \cite{Rag11}):\footnote{This form of the bound is not stated explicitly in \cite{Rag11}.  However, Eq.~(27) of \cite{Rag11} states that $I(V;\xv,\yv) \le \sum_{t=1}^T I(V;y_t|x_1^t, y_1^{t-1})$, where $x_1^t = (x_1,\dotsc,x_t)$ and similarly for $y_1^{t-1}$.  Letting $H(\cdot)$ denote the (differential) entropy function \cite{Cov01}, we obtain \eqref{eq:chain_rule} by writing $I(V;y_t|x_1^t, y_1^{t-1}) = H(y_t|x_1^t, y_1^{t-1}) - H(y_t|x_1^t, y_1^{t-1}, V)$, applying $H(y_t|x_1^t, y_1^{t-1}) \le H(y_t | x_t)$ since conditioning reduces entropy, and applying $H(y_t|x_1^t, y_1^{t-1}, V) = H(y_t|x_t, V)$ since in our setting $y_t$ depends on $(x_1^t, y_1^{t-1}, V)$ only through $(x_t,V)$.}
\begin{equation}
I(V;\xv,\yv) \le \sum_{t=1}^T I(V;y_t|x_t). \label{eq:chain_rule}
\end{equation}
It is well known that the conditional mutual information $I(V;y_t|x_t = x)$ is upper bounded by the maximum KL divergence $\max_{v,v'} D(P_{Y|V,X}(\cdot \,|\, v,x) \| P_{Y|V,X}(\cdot \,|\, v',x))$ between the resulting conditional output distributions $P_{Y|V,X}$ (e.g., see Eq.~(31) of \cite{Rag11}).  In our setting, there are only two values of $v$, and since we are considering Gaussian noise, their conditional output distributions are $N( r_+(x), \sigma^2 )$ and $N( r_-(x), \sigma^2 )$.  Using the standard property that the KL divergence between the $N(\mu_1,\sigma^2)$ and $N(\mu_2,\sigma^2)$ density functions is $\frac{(\mu_2 - \mu_1)^2}{2\sigma^2}$, we deduce that
\begin{equation} 
I(V;y_t|x_t = x) \le \frac{(r_+(x) - r_-(x))^2}{2\sigma^2}. 
\end{equation}
Substituting property \eqref{eq:prop2} in Lemma \ref{lem:r_properties} gives 
\begin{align}
I(V;y_t|x_t = x) 
&\le \frac{(c')^2}{2\sigma^2} \big( \Delta |x - x_0^*| + \Delta^2 \big)^2 \\
&\le \frac{3(c')^2}{2\sigma^2} \big( \Delta^2 |x - x_0^*|^2 + \Delta^4 \big), \label{eq:mi_bound3}
\end{align}
where \eqref{eq:mi_bound3} follows since $(a+b)^2 \le 3(a^2 + b^2)$.  Averaging over $x_t$, we obtain $I(X;y_t|x_t) \le \frac{3(c')^2}{2\sigma^2} \big( \Delta^2 \EE\big[ |x_t - x_0^*|^2 \big] + \Delta^4 \big)$, and substitution into \eqref{eq:chain_rule} gives
\begin{align}
I(V;\xv,\yv) \le \frac{3(c')^2}{2\sigma^2} \bigg( \Delta^2 \EE\bigg[ \sum_{t=1}^T |x_t - x_0^*|^2 \bigg] + T\Delta^4 \bigg). \label{eq:MI}
\end{align}

{\bf Bounding the regret.} We consider the cases $\EE[R_T] \ge c'' T\Delta^2$ and $\EE[R_T] < c'' T\Delta^2$  separately, where $c''$ is defined in Lemma \ref{lem:r_properties}.  In the former case, we immediately have a lower bound on the average cumulative regret, whereas in the latter case, the following lemma is useful.

\begin{lem} \label{lem:contr}
    If $\EE[R_T] < c'' T\Delta^2$ with $c''$ defined in Lemma \ref{lem:r_properties}, then $\EE\big[ \sum_{t=1}^T |x_t-x_0^*|^2 \big] < 4T\Delta^2$.
\end{lem}
\begin{proof}
    Since $V$ is equiprobable on $\{+,-\}$, we have
    \begin{align}
    \EE[R_T]
    &= \sum_{v \in \{+,-\}} \EE\bigg[ \sum_{t=1}^T r_{v}(x_t) \,\Big|\, V=v\bigg] \\
    &\ge c'' \sum_{v \in \{+,-\}} \EE\bigg[\sum_{t=1}^T ( (x_t - x_0^*) + v\Delta)^2 \,\Big|\, V=v\bigg] \label{eq:contr_step2} \\
    % &= \EE\bigg[\sum_{t} x_t^2 + T\Delta^2\bigg] + \frac{1}{2}\EE\bigg[\sum_{t} 2vx_t \Big| V = -1\bigg]  - \frac{1}{2}\EE\bigg[\sum_{t} 2vx_t \Big| V = +1\bigg] \\
    &\ge c'' \EE\bigg[\sum_{t=1}^T (x_t-x_0^*)^2 - 2\sum_{t=1}^T |x_t - x_0^*|\Delta + T\Delta^2\bigg], \label{eq:contr_step3} 
    \end{align}
    where \eqref{eq:contr_step2} follows from \eqref{eq:quadratic_regret} in Lemma \ref{lem:r_properties}, and \eqref{eq:contr_step3} follows by expanding the square and lower bounding the cross-term by its negative absolute value.
    
    Substituting the assumption $\EE[R_T] < c'' T\Delta^2$ into \eqref{eq:contr_step3}, and canceling the term $c'' T \Delta^2$ appearing on both sides, we obtain
    \begin{align}
    \EE\bigg[\sum_{t=1}^T (x_t - x_0^*)^2\bigg] 
    &< 2\Delta\EE\bigg[\sum_{t=1}^T |x_t - x_0^*| \bigg] \label{eq:contr_step4}  \\
    &\le 2\Delta\sqrt{T} \EE\Bigg[\sqrt{\sum_{t=1}^T (x_t - x_0^*)^2} \Bigg] \label{eq:contr_step5}  \\
    &\le 2\Delta\sqrt{T} \sqrt{\EE\bigg[\sum_{t=1}^T (x_t - x_0^*)^2 \bigg]}, \label{eq:contr_step6}
    \end{align}
    where \eqref{eq:contr_step5} follows from the Cauchy-Schwartz inequality, and \eqref{eq:contr_step6} follows from Jensen's inequality.  Solving for $\EE\big[ \sum_{t=1}^T (x_t - x_0^*)^2 \big]$ yields the desired claim.
\end{proof}

In the case $\EE[R_T] < c'' T\Delta^2$, we claim that under the choice $\Delta = \big( \frac{\sigma^2}{\Ctil T} \big)^{1/4}$ with a sufficiently large constant $\Ctil$, it holds that $\EE[R_T] \ge \ctil \sigma \sqrt{T}$ for some constant $\ctil$.  Once this is established, combining the two cases with the choice of $\Delta$ gives
\begin{equation}
    \EE[R_T] \ge \min\bigg\{\frac{c''}{\sqrt{\Ctil}}, \ctil \bigg\} \sigma \sqrt{T},  \label{eq:RT_final}
\end{equation}
which yields Theorem \ref{thm:lb}. We also note that by the assumption $\sigma^2 \le c_{\sigma} T^{1-\zeta}$ in Theorem \ref{thm:lb}, we have for sufficiently large $T$ that $\Delta$ is indeed arbitrarily small under the above choice, as was assumed throughout the proof.\footnote{It is safe to assume that $T$ is sufficiently large, since the smaller values of $T$ can be handled by decreasing $C'$ in the theorem statement.}

It only remains to establish the claim stated above \eqref{eq:RT_final} when $\EE[R_T] < c'' \sigma \sqrt{T}$.  By Lemma \ref{lem:contr}, we have $\EE\big[ \sum_{t=1}^T |x_t-x_0^*|^2 \big] < 4T\Delta^2$, and substitution into \eqref{eq:MI} gives 
\begin{equation}
    I(V;\xv,\yv) \le \frac{15(c')^2}{2\sigma^2} T \Delta^4.
\end{equation}
Since $\Delta^4 = \frac{\sigma^2}{\Ctil T}$, we deduce that $I(V;\xv,\yv) \le \frac{\log 2}{4}$ (say) when $\Ctil$ is sufficiently large.  As a result,  \eqref{eq:Fano} gives $\EE[R_T] \ge \cunder_2 T\Delta^2 H_2^{-1}\big(\frac{3 \log 2}{4}\big)$ (note that $H_2^{-1}$ is an increasing function).  Since $\Delta^2 = \sigma \sqrt{ \frac{1}{\Ctil T} }$, we deduce that $\EE[R_T] \ge \ctil \cdot \sigma \sqrt{T}$, where $\ctil = \cunder_2 \sqrt{ \frac{1}{\Ctil} } H_2^{-1}\big(\frac{3 \log 2}{4}\big)$.  This establishes the desired result.

% $\EE[R_T] \ge \ctil \sigma\sqrt{T}$ as long as $\Ctil$ is sufficiently large.  This is in contradiction with our assumption that $\EE[R_T] < \ctil \sigma\sqrt{T}$, and completes the proof of Theorem \ref{thm:lb}.

\end{document}

%% file: preamble.tex
\usepackage[T1]{fontenc}
\usepackage[latin9]{inputenc}
\usepackage{amsthm}
\usepackage{amsmath}
\usepackage{amssymb}
\usepackage{graphicx}
\usepackage{dsfont}
\usepackage{amsmath}
\usepackage{array}
\usepackage{multirow}
\usepackage{caption}
\usepackage{color}

\usepackage{multicol}

\theoremstyle{plain}

\theoremstyle{plain}

\theoremstyle{plain}
\newtheorem{lem}{\protect\lemmaname}
\theoremstyle{plain}
\newtheorem{thm}{\protect\theoremname}
\theoremstyle{plain}
  
\theoremstyle{definition}

\theoremstyle{definition}
\newtheorem{assump}{\protect\assumptionname}
\theoremstyle{definition}

\makeatother
  
\usepackage{babel} 

\providecommand{\claimname}{Claim}
\providecommand{\lemmaname}{Lemma}
\providecommand{\propositionname}{Proposition}
\providecommand{\theoremname}{Theorem}
\providecommand{\corollaryname}{Corollary} 
\providecommand{\definitionname}{Definition}
\providecommand{\assumptionname}{Assumption}
\providecommand{\remarkname}{Remark}

\DeclareMathOperator*{\argmin}{arg\,min}

\newcommand{\cunder}{\underline{c}}
\newcommand{\cbar}{\overline{c}}

\newcommand{\Tearly}{T_{\mathrm{early}}}

\newcommand{\Rearly}{R_{\mathrm{early}}}
\newcommand{\Rlate}{R_{\mathrm{late}}}
\newcommand{\Ctil}{\widetilde{C}}
\newcommand{\ctil}{\widetilde{c}}
\newcommand{\ftil}{\widetilde{f}}
\newcommand{\UCB}{\mathrm{UCB}}
\newcommand{\LCB}{\mathrm{LCB}}
\newcommand{\UCBtil}{\widetilde{\mathrm{UCB}}}
\newcommand{\LCBtil}{\widetilde{\mathrm{LCB}}}
\newcommand{\Lctil}{\widetilde{\mathcal{L}}}

\newcommand{\xv}{\mathbf{x}}

\newcommand{\yv}{\mathbf{y}}

\newcommand{\Ac}{\mathcal{A}}
\newcommand{\Bc}{\mathcal{B}}

\newcommand{\Ic}{\mathcal{I}}

\newcommand{\EE}{\mathbb{E}}
\newcommand{\PP}{\mathbb{P}}

\newcommand{\ZZ}{\mathbb{Z}}

\newcommand{\Lc}{\mathcal{L}}

\newcommand{\Iv}{\mathbf{I}}

\newcommand{\fv}{\mathbf{f}}

\newcommand{\Kv}{\mathbf{K}}
\newcommand{\kv}{\mathbf{k}}